\documentclass[11pt]{article}
\usepackage[margin=1in]{geometry}

\usepackage{etoolbox}
\usepackage{multirow}

\newtoggle{arxiv}

\toggletrue{arxiv}
\usepackage{macros}

\bluehyperref

\usepackage[style=alphabetic,backend=bibtex,maxnames=12,maxbibnames=10,maxcitenames=10,maxalphanames=10,giveninits=true,doi=false,url=true]{biblatex}
\newcommand*{\citet}[1]{\AtNextCite{\AtEachCitekey{\defcounter{maxnames}{2}}} \textcite{#1}}

\newcommand*{\citep}[1]{\cite{#1}}

\bibliography{superbib}

\usepackage{hyperref} 
\usepackage[capitalize]{cleveref}
\usepackage{crossreftools}
\usepackage[T1]{fontenc}

\usepackage[noend]{algorithmic}
\usepackage{algorithm}

\title{Private Adaptive Gradient Methods for Convex Optimization}
\author{%
    Hilal Asi\thanks{The authors are in alphabetical order.} \thanks{Department of Electrical Engineering, Stanford University. Work done while interning at Apple; \texttt{asi@stanford.edu}.}
    \and John Duchi\thanks{Departments of Electrical Engineering and Statistics, Stanford University and Apple; \texttt{jduchi@stanford.edu}.}
    \and Alireza Fallah\thanks{Department of Electrical Engineering \& Computer Science, Massachusetts Institute of Technology. Work done while interning at Apple; \texttt{afallah@mit.edu}.
    }
    \and Omid Javidbakht\thanks{Apple; \texttt{omid\_j@apple.com}.}
    \and Kunal Talwar\thanks{Apple; \texttt{kunal@kunaltalwar.org}.}
    }

\date{}

\begin{document}
\maketitle

\begin{abstract}
  We study adaptive methods for differentially private convex optimization,
  proposing and analyzing differentially private variants of a Stochastic
  Gradient Descent (SGD) algorithm with adaptive stepsizes, as well as the
  AdaGrad algorithm. We provide upper bounds on the regret of both
  algorithms and show that the bounds are (worst-case) optimal. As a
  consequence of our development, we show that our private versions of
  AdaGrad outperform adaptive SGD, which in turn outperforms traditional SGD
  in scenarios with non-isotropic gradients where (non-private) Adagrad
  provably outperforms SGD. The major challenge is that
  the isotropic noise typically added for privacy dominates the signal in
  gradient geometry for high-dimensional problems; approaches to this that
  effectively optimize over lower-dimensional subspaces simply ignore the
  actual problems that varying gradient geometries introduce. In contrast, we
  study non-isotropic clipping and noise addition, developing a
  principled theoretical approach; the consequent procedures
  also enjoy significantly
  stronger empirical
  performance than prior approaches.
\end{abstract}


\section{Introduction}


While the success of stochastic gradient methods for solving empirical risk
minimization has motivated their adoption across much of machine learning,
increasing privacy risks in data-intensive tasks have made applying them
more challenging~\cite{DworkMcNiSm06}: gradients can leak users' data,
intermediate models can compromise individuals, and even final trained
models may be non-private without substantial care. This motivates a growing
line of work developing private variants of stochastic gradient descent
(SGD), where algorithms guarantee differential privacy by perturbing
individual gradients with random noise~\cite{DuchiJoWa13_focs, SmithTh13,
  AbadiChGoMcMiTaZh16, DuchiJoWa18, BassilyFeTaTh19, FeldmanKoTa20}.  Yet
these noise addition procedures typically fail to reflect the geometry
underlying the optimization problem, which in non-private cases is
essential: for high-dimensional problems with sparse parameters, mirror
descent and its variants~\cite{BeckTe03, NemirovskiJuLaSh09} are essential, 
while in the large-scale stochastic settings prevalent in deep learning,
AdaGrad and other adaptive variants~\cite{DuchiHaSi11} provide stronger
theoretical and practical performance. Even more, methods that do not adapt 
(or do not leverage geometry) can be provably sub-optimal, in that there
exist problems where their convergence is much slower than adaptive variants
that reflect appropriate geometry~\cite{LevyDu19}.

To address these challenges, we introduce \PAGAN~(Private AdaGrad with
Adaptive Noise), a new differentially private variant of stochastic gradient
descent and AdaGrad. Our main contributions center on a few ideas. Standard
methods for privatizing adaptive algorithms that add isometric (typically
Gaussian) noise to gradients necessarily reflect the worst-case behavior of
functions to be optimized and eliminate the geometric structure one might
leverage for improved convergence. By carefully adapting noise to the actual
gradients at hand, we can both achieve convergence rates that reflect the
observed magnitude of the gradients---similar to the approach of
\citet{BartlettHaRa07} in the non-private case---which can yield marked
improvements over the typical guarantees that depend on worst-case
magnitudes. (Think, for example, of a standard normal variable: its second
moment is 1, while its maximum value is unbounded.)  Moreover, we propose a
new private adaptive optimization algorithm that analogizes AdaGrad, showing
that under certain natural distributional assumptions for the
problems---similar to those that separate AdaGrad from non-adaptive
methods~\cite{LevyDu19}---our private versions of adaptive methods
significantly outperform the standard non-adaptive private algorithms.
Additionally, we prove several lower bounds that both highlight the
importance of geometry in the problems and demonstrate the tightness of the
bounds our algorithms achieve.
Finally, we provide several experiments on real-world and synthetic datasets that support our theoretical results, demonstrating the improvements of our private adaptive algorithm (\PAGAN) over DP-SGD and other private adaptive methods.

\subsection{Related Work}

Since the introduction of differential privacy~\cite{DworkMcNiSm06,
  DworkKeMcMiNa06}, differentially private empirical risk minimization has
been a subject of intense interest~\cite{ChaudhuriMoSa11, BassilySmTh14,
  DuchiJoWa13_focs, SmithTh13lasso}.  The current standard approach to
solving this problem is noisy SGD~\cite{BassilySmTh14, DuchiJoWa13_focs,
  AbadiChGoMcMiTaZh16, BassilyFeTaTh19, FeldmanKoTa20}.  Current bounds
focus on the standard Euclidean geometries familiar from classical analyses
of gradient descent~\cite{Zinkevich03, NemirovskiJuLaSh09}, and the
prototypical result~\cite{BassilySmTh14, BassilyFeTaTh19} is that, for
Lipschitz convex optimization problems on the $\ell_2$-ball in
$d$-dimensions, an $\diffp$-differentially private version of SGD achieves
excess empirical loss $O(\frac{\sqrt{d}}{n \diffp})$ given a sample of size
$n$; this is minimax optimal.  
Similar bounds also hold for other geometries ($\ell_p$-balls for $1 \le p \le 2$) using noisy mirror descent~\cite{AsiFeKoTa21}.
Alternative approaches use the stability of
empirical risk minimizers of (strongly) convex functions, and include both
output perturbation, where one adds noise to a regularized empirical
minimizer, and objective perturbation, where one incorporates random linear
noise in the objective function before optimization~\cite{ChaudhuriMoSa11}. 

Given the success of private SGD for such Euclidean cases and adaptive
gradient algorithms for modern large-scale learning, it is unsurprising that
recent work attempts to incorporate adaptivity into private empirical risk
minimization (ERM) algorithms~\cite{ZhouWuBa20, KairouzRiRuTh20}.  In this
vein, \citet{ZhouWuBa20} propose a private SGD algorithm where the gradients
are projected to a low-dimensional subspace---which is learned using public
data---and~\citet{KairouzRiRuTh20} developed an $\diffp$-differentially
private variant of Adagrad which (similarly) projects the gradient to a low
rank subspace. These works show that excess loss $\widetilde O (\frac{1}{n
  \diffp}) \ll \frac{\sqrt{d}}{n \diffp}$ is possible whenever the rank of
the gradients is small.  Yet these both work under the assumption that
gradient lie in (or nearly in) a low-dimensional subspace; this misses the
contexts for which adaptive algorithms (AdaGrad and its relations) are
designed~\cite{DuchiHaSi11, McMahanSt10}.  Indeed, most stochastic
optimization algorithms rely on particular dualities between the parameter
space and gradients; stochastic gradient descent requires Euclidean spaces,
while mirror descent works in an $\ell_1/\ell_\infty$ duality (that is, it
is minimax optimal when optimizing over an $\ell_1$-ball while gradients
belong to an $\ell_\infty$ ball). AdaGrad and other adaptive algorithms, in
contrast, are optimal in an (essentially) dual geometry~\cite{LevyDu19}, so
that for such algorithms, the interesting geometry is when the parameters
belong (e.g.) to an $\ell_\infty$ box and the gradients are sparse---but
potentially from a very high-rank space. Indeed, as Levy and
Duchi~\cite{LevyDu19} show, adaptive algorithms achieve benefits only when
the sets over which one optimizes are quite different from $\ell_2$ balls;
the private projection algorithms in the papers by Kairouz et
al.~\cite{KairouzRiRuTh20} and Zhou et al.~\cite{ZhouWuBa20} achieve bounds
that scale with the $\ell_2$-radius of the underlying space, suggesting that
they may not enjoy the performance gains one might hope to achieve using an
appropriately constructed and analyzed adaptive algorithm.

In more recent work, Yu et al.~\cite{YuZCL21} use PCA to decompose gradients into two
orthogonal subspaces, allowing separate learning rate treatments in the
subspaces, and achieve promising empirical results, but they provide no
provable convergence bounds.
Also related to the current paper is Pichapati et al.'s
AdaClip algorithm~\cite{PichapatiSYRK20}; they obtain parallels to
Bartlett et al.'s
non-private convergence guarantees~\cite{BartlettHaRa07}
for private SGD.  In contrast to our analysis
here, their analysis applies to smooth non-convex functions, while our focus
on convex optimization allows more complete convergence guarantees and
associated optimality results.



\section{Preliminaries and notation}

Before proceeding to the paper proper, we give notation.  Let $\domain$ be a
sample space and $P$ a distribution on $\mc{Z}$.  Given a function $F :
\xdomain \times \domain \to \R$, convex in its first argument, and a dataset
$\Ds = (\ds_1,\dots,\ds_n) \in \domain^n$ of $n$ points drawn i.i.d.\ $P$,
we wish to privately find the minimizer of the empirical loss
\begin{equation}
  \label{eqn:emp_loss}
  \argmin_{x \in \xdomain} f(x;\Ds) \defeq
  \frac{1}{n} \sum_{i=1}^n F(x;\ds_i).
\end{equation}
We suppress dependence on $\Ds$ and simply write $f(x)$ when the
dataset is clear from context.  We use the standard definitions of
differential privacy~\cite{DworkMcNiSm06, DworkKeMcMiNa06}:
\begin{definition}
  \label{def:DP}
  A randomized algorithm $\mech$ is
  \emph{$(\diffp,\delta)$-differentially private} if for all neighboring
  datasets
  $\Ds,\Ds' \in \domain^n$ and all
  measurable $O$ in the output
  space of $\mech$,
  \begin{equation*}
    \P\left(\mech(\Ds) \in O \right)
    \le e^{\diffp} \P\left(\mech(\Ds') \in O \right) + \delta.
  \end{equation*}
  If $\delta=0$, then $\mech$ is \emph{$\diffp$-differentially private}.
\end{definition}



\noindent
It will also be useful to discuss the tail properties of random variables
and vectors:
\begin{definition}\label{def:sub-Gaussian}
  A random variable $X$ is \emph{$\sigma^2$
    sub-Gaussian} if $\E[\exp(s(X - \E[X]))]
  \leq \exp((\sigma^2 s^2)/2)$ for all $s
  \in \reals$. A random vector $X \in \reals^d$ is
  $\Sigma$-sub-Gaussian if for any vector $a \in
  \reals^d$, $a^\top X$ is $a^\top \Sigma a$ sub-Gaussian.
\end{definition}

\noindent
We also frequently use different norms and geometries,
so it is useful to recall Lipschitz continuity:
\begin{definition}
  \label{def:Lipschitz_function}
  A function $\Phi : \R^d \to \R$ is \emph{$G$-Lipschitz with respect to
    norm $\norm{\cdot}$ over $\mathcal{W}$} if for every $w_1,w_2 \in
  \mathcal{W}$,
  \begin{equation*}
    |\Phi(w_1) - \Phi(w_2) | \leq G \norm{w_1 - w_2}.
  \end{equation*}
\end{definition}
\noindent
A convex function $\Phi$ is $G$-Lipschitz over an open set
$\mathcal{W}$ if and only if $\| \Phi'(w) \|_{*} \leq G$ for any $w \in
\mathcal{W}$ and $\Phi'(w) \in \partial \Phi(w)$, where
$\dnorm{y} = \sup\{x^\top y \mid \norm{x} \le 1\}$ is the
dual norm of $\norm{\cdot}$~\citep{HiriartUrrutyLe93ab}.

\paragraph{Notation}

We define $\diag(a_1,\ldots,a_d)$ as a diagonal matrix with diagonal entries
$a_1,\ldots,a_d$. To state matrix $A$ is positive (semi)definite, we use the
notation $A \succ 0_{d\times d}$ ($A \succcurlyeq 0_{d \times d}$). For $A
\succcurlyeq 0_{d \times d}$, let $E_A$ denote the ellipsoid $\{x: \normA{x}
\leq 1 \}$ where $\normA{x} = \sqrt{x^\top A x}$ is the Mahalanobis norm,
and $\pi_{A}(x) = \argmin_y\{\ltwo{y - x} \mid y \in E_A\}$ is the
projection of $x$ onto $E_A$.  For a set $\xdomain$,
$\diam_{\norm{\cdot}}(\xdomain) = \sup_{x, y \in \xdomain} \norm{x - y}$
denotes the diameter of $\xdomain$ with respect to the norm
$\norm{\cdot}$. For the special case of $\norm{\cdot}_p$, we write
$\diam_p(\xdomain)$ for simplicity. For an integer $n \in \N$, we let $[n] =
\{1,\dots,n\}$.

\section{Private Adaptive Gradient Methods} \label{sec:algs}

In this section, we study and develop \PASAN and \PAGAN, differentially private versions of
Stochastic Gradient Descent (SGD) with adaptive stepsize (Algorithm
\ref{Algorithm1}) and Adagrad \cite{DuchiHaSi11} (Algorithm
\ref{Algorithm2}).  The challenge in making these algorithms private is that
adding isometric Gaussian noise---as is standard in the differentially
private optimization literature---completely eliminates the geometrical
properties that are crucial for the performance of adaptive gradient
methods. We thus add noise that adapts to gradient geometry while
maintaining privacy.
More precisely, our private versions of adaptive optimization algorithms
proceed as follows: to privatize the gradients, we first project them to an
ellipsoid capturing their geometry, then adding non-isometric Gaussian noise
whose covariance corresponds to the positive definite matrix $A$ that
defines the ellipsoid. Finally, we apply the adaptive algorithm's step with
the private gradients. We present our private versions of SGD with adaptive
stepsizes and Adagrad in Algorithms~\ref{Algorithm1} and~\ref{Algorithm2},
respectively.

\begin{algorithm}[tb]
	\caption{Private Adaptive SGD with Adaptive Noise (\PASAN)}
	\label{Algorithm1}
	\begin{algorithmic}[1]
	\REQUIRE Dataset $\mathcal{S} = (\ds_1,\dots,\ds_n) \in \domain^n$, convex set $\xdomain$, mini-batch size $\batch$, number of iterations $T$, privacy parameters $\diffp, \delta$;
	\STATE Choose arbitrary initial point $x^0 \in \xdomain$;
    \FOR{$k=0$ to $T-1$\,}
        \STATE Sample a batch $\mathcal{D}_k:= \{z_i^k\}_{i=1}^\batch$ from $\mathcal{S}$ uniformly with replacement;
        \STATE Choose ellipsoid $A_k$;
        \STATE Set $\tilde g^k := \frac{1}{\batch} \sum_{i=1}^\batch \pi_{A_k}(g^{k,i})$ where $g^{k,i} \in \partial F(x^k; z_i^k)$;
        \STATE Set $ \hg^k = \tilde g^k + \sqrt{\log(1/\delta)}/{(\batch \diffp)} \noise^k$ where $\noise^k \simiid \normal(0, A_k^{-1})$;
        \STATE Set $\stepsize_k = \alpha/{\sqrt{\sum_{i = 0}^k \ltwos{\hg^i}^2}}$;
        \STATE $x^{k+1} :=  \proj_{\xdomain}( x^k -  \alpha_k \hg^k)$;
    \ENDFOR
        
    \STATE {\bfseries Return:} $\wb{x}^T \defeq \frac{1}{T} \sum_{k = 1}^T x^k$
	\end{algorithmic}
\end{algorithm}

\begin{algorithm}[tb]
	\caption{Private Adagrad with Adaptive Noise (\PAGAN)}
	\label{Algorithm2}
	\begin{algorithmic}[1]
	\REQUIRE Dataset $\mathcal{S} = (\ds_1,\dots,\ds_n) \in \domain^n$, convex set $\xdomain$, mini-batch size $\batch$, number of iterations $T$, privacy parameters $\diffp, \delta$;
	\STATE Choose arbitrary initial point $x^0 \in \xdomain$;
    \FOR{$k=0$ to $T-1$\,}
        \STATE Sample a batch $\mathcal{D}_k:= \{z_i^k\}_{i=1}^\batch$ from $\mathcal{S}$ uniformly with replacement;
        \STATE Choose ellipsoid $A_k$;
        \STATE Set $\tilde g^k := \frac{1}{\batch} \sum_{i=1}^\batch \pi_{A_k}(g^{k,i})$ where $g^{k,i} \in \partial F(x^k; z_i^k)$;
        \STATE Set $ \hg^k = \tilde g^k + \sqrt{\log(1/\delta)}/{(\batch \diffp)} \noise^k$ where $\noise^k \simiid \normal(0, A_k^{-1})$;
        \STATE Set $H_k = \diag\left(\sum_{i=0}^{k} \hg^i \hg^{i^T}  \right)^{\frac{1}{2}}/ \alpha$;
        \STATE $x^{k+1} =  \proj_{\xdomain}( x^k -  H_k^{-1}  \hg^k)$ where the projection is with respect to $\norm{\cdot}_{H_k}$;
    \ENDFOR
        
    \STATE {\bfseries Return:} $\wb{x}^T \defeq \frac{1}{T} \sum_{k = 1}^T x^k$
	\end{algorithmic}
\end{algorithm}

Before analyzing the utility of these algorithms, we provide their 
privacy guarantees in the following lemma (see Appendix \ref{sec:proof-lemma:privacy} for its proof).
\begin{lemma}
  \label{lemma:privacy}
  There exist constants $\bar{\diffp}$ and $c$ such that, for any
  $\diffp \leq \bar{\diffp}$, and with $T = c{n^2}/{b^2}$, 
  Algorithm~\ref{Algorithm1} and Algorithm~\ref{Algorithm2} are
  $(\diffp,\delta)$-differentially private.
\end{lemma}

Having established the privacy guarantees of our algorithms, we now proceed
to demonstrate their performance. To do so, we introduce an
assumption to that, as we shall see presently, will allow
us to work in gradient geometries different than the classical Euclidean
($\ell_2$) one common to current private optimization analyses.
\begin{assumption}
  \label{assumption:Lipschitz_C}
  There exists a function
  $\lipobj : \mc{Z} \times \R^{d \times d} \to \R_+$ such
  that for any diagonal $C \succ 0$,
  the function $F(\cdot; z)$ is
  $\lipobj(z;C)$-Lipschitz with respect to the Mahalanobis norm
  $\|.\|_{C^{-1}}$ over $\xdomain$, i.e., $\norm{\nabla f(x;z)}_C \le
  \lipobj(z;C) $ for all $x \in \xdomain$.
\end{assumption}
\noindent
The moments of the Lipschitz constant $G$ will be central
to our convergence analyses, and to that end,
for $p \ge 1$ we define the shorthand
\begin{equation}  
  \label{eqn:moment-matrix-norm}
  \lipobj_p(C):= 
  \E_{z \sim P} \left[ \lipobj(z;C)^p \right]^{1/p}.
\end{equation}
The quantity $\lipobj_p(C)$ are the $p$th moments of the gradients in the
Mahalanobis norm $\norm{\cdot}_C$; they are the key to our stronger
convergence guarantees and govern the error in projecting our gradients.  In
most standard analyses of private optimization (and stochastic optimization
more broadly), one takes $C = I$ and $p = \infty$, corresponding to the
assumption that $F(\cdot, z)$ is $\lipobj$-Lipschitz for all $z$ and that
subgradients $F'(x, z)$ are uniformly bounded in both $x$ and $z$. Even when
this is the case---which may be unrealistic---we always have $\lipobj_p(C)
\le \lipobj_\infty(C)$, and in realistic settings there is often a
significant gap; by depending instead on appropriate moments $p$, we shall
see it is often possible to achieve far better convergence guarantees than
would be possible by relying on uniformly bounded moments.  (See also
Barber and Duchi's
discussion of these issues in the context of mean
estimation~\cite{BarberDu14a}.)


An example may be clarifying:

\begin{example}\label{example-sub-Gaussian}
  Let $g:\reals^d \to \reals$ be a convex and differentiable function, let
  $F(x;Z) = g(x) + \<x, Z\>$ where $Z \in \R^d$ and the coordinates $Z_j$
  are independent $\sigma_j^2$-subgaussian, and $C \succ 0$ be diagonal.
  Then by standard moment bounds (see
  Appendix~\ref{sec:proof-example-sub-Gaussian}), if $\|\nabla g(x)\|_C \leq
  \mu$ we have
  \begin{equation}\label{eqn:lip-sub-Gaussian}
    \lipobj_p(C) 
    \le \mu + O(1) \sqrt{p} \sqrt{\sum_{j=1}^d C_{jj} \sigma_j^2}.  
  \end{equation}
  As this
  bound shows, while $\lipobj_\infty$ is infinite in this example,
  $\lipobj_p$ is finite. As a result, our analysis extends to
  settings in which the stochastic gradients are not uniformly bounded.
\end{example}

While we defined $\lipobj_p(C)$ by taking expectation with respect to the
original distribution $P$, we mainly focus on empirical risk minimization
and thus require the empirical Lipschitz
constant for a given dataset $\Ds$:
\begin{equation} \label{eqn:empirical_lipobj}
  \hat{\lipobj}_p(\Ds;C) := \left(\frac{1}{n} \sum_{i=1}^n \lipobj(z_i;C)^p
  \right)^{1/p}.
\end{equation}
A calculation using Chebyshev's inequality and that $p$-norms are increasing
immediately gives the next lemma:
\begin{lemma}
\label{lemma:empirical_Lipschitz}
 Let $\Ds$ be a dataset with $n$ points sampled from distribution $P$.
 Then with probability at least $1-1/n$, we have
 \begin{equation*}
   \hat{\lipobj}_p(\Ds;C) \leq  \lipobj_p(C) + \lipobj_{2p}(C) \leq 2 \lipobj_{2p}(C), 
 \end{equation*}
\end{lemma}
\noindent
It is possible to get bounds of the form
$\hat{\lipobj}_p(\Ds; C) \lesssim \lipobj_{kp}(C)$ with probability at least
$1 - 1/n^k$ using Khintchine's inequalities, but this is secondary for us.


Given these moment bounds, we can characterize the convergence of both
algorithms, defering proofs to Appendix~\ref{sec:proofs-alg}.
\subsection{Convergence of {\PASAN}}
\label{sec:ub-sgd}
We first start with \PASAN (Algorithm
\ref{Algorithm1}).  Similarly to the non-private setting where SGD (and its
adaptive variant) are most appropriate for domains $\xdomain$ with small
$\ell_2$-diameter $\diam_2(\xdomain)$, our bounds in this section
mostly depend on $\diam_2(\xdomain)$.
\begin{theorem}
  \label{theorem:convergence-SGD}
  Let $\Ds \in \mc{Z}^n$ and $C \succ 0$ be diagonal, $p \ge 1$, and assume
  that $\hat{\lipobj}_p(\Ds; C) \le \lipobj_{2p}(C)$.  Consider running \PASAN
  (Algorithm \ref{Algorithm1}) with $\alpha=\diam_2(\xdomain)$, $T=cn^2/b^2$, $A_k = \frac{1}{B^2} C$,
  where\footnote{We provide the general statement of this theorem for
    positive $B$ in Appendix \ref{proof-theorem:convergence-SGD}}
  \begin{equation*}
    B =
    2\lipobj_{2p}(C) \left( \frac{\diam_{\norm{\cdot}_{C^{-1}}}(\xdomain) n
      \diffp}{\diam_2(\xdomain) \sqrt{\tr(C^{-1})} \sqrt{\log(1/\delta)}}
    \right )^{1/p}
  \end{equation*}
  and $c$ is the constant in Lemma~\ref{lemma:privacy}. Then
  \begin{align*}
    & \E [f(\wb{x}^T;\Ds) - \min_{x \in \xdomain} f(x;\Ds)] \\
    & \leq  \bigO \left( \frac{\diam_2(\xdomain)}{T} \sqrt{\sum_{k=1}^T \E [\ltwos{g^k}^2] } 
  +  \diam_2(\xdomain) \lipobj_{2p}(C) \times  \right. \\ 
    & \left.  \;\;\;\;\;~  \left( \frac{\sqrt{\tr(C^{-1}) \ln\tfrac 1 \delta} }{n \diffp} \right)^{\frac{p-1}{p}} \!\! \!
    \left( \frac{\diam_{\norm{\cdot}_{C^{-1}}}(\xdomain)}{\diam_2(\xdomain)} \right)^{\frac{1}{p}} \right ),
  \end{align*}
  where the expectation is taken over the internal randomness of the algorithm.
\end{theorem}

\newcommand{\stdregret}{R_{\textup{std}}(T)}
\newcommand{\adaregret}{R_{\textup{ada}}(T)}

To gain intuition for these
bounds, note that for large enough $p$, the bound from Theorem
\ref{theorem:convergence-SGD} is approximately
\iftoggle{arxiv}{
\begin{equation}\label{eqn:SGD-approx}
  \diam_2(\xdomain) \Bigg(
  \underbrace{\frac{1}{T} \sqrt{\sum_{k=1}^T \E [\ltwos{g^k}^2] }}_{=:
    \stdregret}  + \lipobj_{2p}(C) \cdot \frac{\sqrt{\tr(C^{-1}) \log(1/\delta)} }{n \diffp}
  \Bigg). 
\end{equation}
}
{
\begin{equation}\label{eqn:SGD-approx}
\begin{split}
  \diam_2(\xdomain) \Bigg(
  & \underbrace{\frac{1}{T} \sqrt{\sum_{k=1}^T \E [\ltwos{g^k}^2] }}_{=:
    \stdregret} \\
  & + \lipobj_{2p}(C) \cdot \frac{\sqrt{\tr(C^{-1}) \log(1/\delta)} }{n \diffp}
  \Bigg). 
\end{split}
\end{equation}
}
The term $\stdregret$ in~\eqref{eqn:SGD-approx} is the standard non-private
convergence rate for SGD with adaptive stepsizes~\cite{BartlettHaRa07,
  Duchi18} and (in a minimax sense) is unimprovable even without
privacy; the second term is the
cost of privacy.  In the standard setting of gradients uniformly bounded in
$\ell_2$-norm, where $C = I$ and $p=\infty$, this bound recovers the
standard rate $\diam_2(\xdomain) \lipobj_\infty(I) \frac{\sqrt{d
    \log(1/\delta)}}{n \diffp}$.  However, as we show in our examples, this
bound can offer significant improvements whenever $C \neq I$ such that
$\tr(C^{-1}) \ll d$ or $\lipobj_{2p}(C) \ll \lipobj_\infty$ for some $p <
\infty$.





\subsection{Convergence of \PAGAN}
\label{sec:ub-adagrad}
Having established our bounds for \PASAN, we now proceed to present our results for \PAGAN (Algorithm \ref{Algorithm2}). 
In the non-private setting, adaptive gradient methods such as Adagrad are superior to SGD for constraint sets 
such as $\xdomain = [-1,1]^d$ where $\diam_\infty(\xdomain) \ll \diam_2(\xdomain)$. Following this, our bounds in this section will depend on $\diam_\infty(\xdomain)$.

\begin{theorem}\label{theorem:private-adagrad}
  Let $\Ds \in \mc{Z}^n$ and $C \succ 0$ be diagonal, $p \ge 1$, and assume
  that $\hat{\lipobj}_p(\Ds; C) \le \lipobj_{2p}(C)$.  Consider running \PAGAN
  (Algorithm~\ref{Algorithm2}) with $\alpha = \diam_\infty(\xdomain)$, $T=cn^2/b^2$, $A_k = \frac{1}{B^2} C$,
  where  
  \begin{equation*}
    B = 2\lipobj_{2p}(C)  \left( \frac{\diam_{\norm{\cdot}_{C^{-1}}}(\xdomain) n \diffp}{\diam_\infty(\xdomain) \sqrt{\log(1/\delta)} \tr(C^{-\half})} \right)^{1/p}
  \end{equation*}
  and $c$ is the constant in Lemma \ref{lemma:privacy}. Then
  \begin{align*}
    & \E [f(\wb{x}^T;\Ds) - \min_{x \in \xdomain} f(x;\Ds)] \\
    & \leq \! \bigO \left(   \frac{\diam_\infty(\xdomain)}{T} \sum_{j=1}^d 
    \E \left [ \sqrt{\sum_{k=1}^T (g^{k}_j)^2} \right ] 
    \! + \diam_\infty(\xdomain) \times
    \right. \\ 
    & \left. \! \;\;\;\lipobj_{2p}(C) 
    \left( \frac{\sqrt{\ln\tfrac 1 \delta} \tr(C^{-\half})}{n \diffp} \right)^{\frac{p-1}{p}} \!\!\!
    \left( \frac{\diam_{\norm{\cdot}_{C^{-1}}}(\xdomain)}{\diam_\infty(\xdomain)} \right)^{\frac{1}{p}} \right),
  \end{align*}
  where the expectation is taken over the internal randomness of the algorithm.
\end{theorem}

To gain intuition, we again consider the large $p$ case, where
Theorem~\ref{theorem:private-adagrad} simplifies to roughly
\iftoggle{arxiv}{
\begin{equation*}
\diam_\infty(\xdomain) \Bigg(
 \underbrace{\frac{1}{T} \sum_{j=1}^d 
    \E \left [ \sqrt{\sum_{k=1}^T (g^{k}_j)^2} \right ]}_{=:
    \adaregret} + \lipobj_{2p}(C) \left( \frac{\sqrt{\log(1/\delta)} \tr(C^{-\half})}{n \diffp} \right) \Bigg).
\end{equation*}
}
{
\begin{equation*}
\begin{split}
\diam_\infty(\xdomain) \Bigg(
 & \underbrace{\frac{1}{T} \sum_{j=1}^d 
    \E \left [ \sqrt{\sum_{k=1}^T (g^{k}_j)^2} \right ]}_{=:
    \adaregret} \\
&	+ \lipobj_{2p}(C) \left( \frac{\sqrt{\log(1/\delta)} \tr(C^{-\half})}{n \diffp} \right) \Bigg).
\end{split}
\end{equation*}
}
In analogy with Theorem~\ref{theorem:convergence-SGD}, the first term
$\adaregret$ is the
standard error for non-private Adagrad after $T$
iterations~\cite{DuchiHaSi11}---and hence
unimprovable~\cite{LevyDu19}---while the second is the privacy cost. In
some cases, we may have
$\diam_\infty(\xdomain) = \diam_2(\xdomain) / \sqrt{d}$, so private
Adagrad can offer significant improvements over SGD whenever the matrix $C$
has polynomially decaying diagonal.

To clarify the advantages and scalings we expect,
we may consider an extremely stylized example with
sub-Gaussian distributions. Assume now---in the context
of Example~\ref{example-sub-Gaussian}---that we are
optimizing the random linear function $F(x; Z) = \<x, Z\>$,
where $Z$ has independent $\sigma_j^2$-sub-Gaussian compoments.
In this case, by assuming that $p = \log d$ and taking
$C_{jj} = \sigma_j^{-4/3}$ and $b = 1$,
Theorem~\ref{theorem:private-adagrad} guarantees that
\PAGAN~(Algorithm~\ref{Algorithm2}) has convergence
\begin{align}
  \label{eqn:pagan-subgaussian-bound}
 & \E [f(\wb{x}^T;\Ds) - \min_{x \in \xdomain} f(x;\Ds)]
  \leq \iftoggle{arxiv}{}{\\
  &} \bigO(1)
  \diam_\infty(\xdomain)  \bigg[ \adaregret 
    + \frac{ (\sum_{j=1}^d \sigma_j^{2/3})^{3/2} }{n \diffp}
    \log \frac{d}{\delta}\bigg].
\end{align}

On the other hand, for \PASAN~(Algorithm~\ref{Algorithm1}),
with $p = \log d$, $b=1$, the choice
$C_{jj} = \sigma_j^{-1}$ optimizes the bound
of Theorem~\ref{theorem:convergence-SGD} and yields
\begin{align}\label{eqn:pasan-subgaussian-bound-pasan}
 & \E [f(\wb{x}^T;\Ds) - \min_{x \in \xdomain} f(x;\Ds)]
  \leq \iftoggle{arxiv}{}{\\
  &} \bigO(1)    
  \diam_2(\xdomain)  \left[ \stdregret 
    + \frac{ \sum_{j=1}^d \sigma_j }{n \diffp}
    \log \frac{d}{\delta}\right].
\end{align}

Comparing these results, two differences are salient:
$\diam_\infty(\xdomain)$ replaces
$\diam_2(\xdomain)$ in Eq.~\eqref{eqn:pasan-subgaussian-bound-pasan}, which
can be an improvement by as much as $\sqrt{d}$, while
$(\sum_{j=1}^d \sigma_j
^{2/3})^{3/2}$ replaces $\sum_{j=1}^d
\sigma_j$, and H\"{o}lder's inequality gives
\begin{equation*}
  \sqrt{d} \sum_{j=1}^d \sigma_j \geq \left(\sum_{j=1}^d \sigma_j ^{2/3}\right)^{3/2} \geq \sum_{j=1}^d \sigma_j.
\end{equation*}
Depending on gradient moments, there are situations in which
\PAGAN offers significant improvements; these evidently
depend on the expected magnitudes of the gradients and noise,
as the $\sigma_j$ terms evidence. As a special case, consider
$\xdomain = [-1,+1]^d$ and assume $\{\sigma_j\}_{j=1}^d$ decrease
quickly, e.g.\ $\sigma_j = 1/j^{3/2}$. In such a
setting, the upper bound of \PAGAN is roughly
$\frac{\mathsf{poly}(\log d)}{n \diffp}$ while \PASAN 
achieves $\frac{\sqrt{d}}{n \diffp}$.

\section{Some approaches to unknown moments}
\label{sec:unknown-cov}

As the results of the previous section demonstrate,
bounding the gradient moments allows us to establish tighter
convergence guarantees; it behooves us to
estimate them with accuracy sufficient to achieve
(minimax) optimal bounds.

\subsection{Unknown moments for generalized linear models}
\label{sec:cov-GLM}

Motivated by the standard practice of training the last layer of a
pre-trained neural network~\cite{AbadiChGoMcMiTaZh16}, in this section we
consider algorithms for generalized linear models, where we have losses of
the form $F(x;z) = \ell(z^T x)$ for $z,x\in \R^d$ and $\ell : \R \to \R_+$
is a convex and $1$-Lipschitz loss.  As $\nabla F(x;z) = \ell'(z^T x) z$,
bounds on the Lipschitzian moments~\eqref{eqn:moment-matrix-norm}
follow from moment bounds on $z$ itself, as $\norm{\nabla F(x; z)} \le
\norm{z}$.

The results of Section \ref{sec:algs} suggest optimal choices for
$C$ under sub-Gaussian assumptions on the vectors $z$, where in our stylized
cases of $\sigma_j$-sub-Gaussian entries, $C_j = \sigma_j^{-4/3}$
minimizes our bounds. Unfortunately, it is hard in general to estimate
$\sigma_j$ even without privacy~\cite{Duchi19}. Therefore, we make the
following bounded moments ratio assumption, which relates higher moments to
lower moments to allow estimation of moment-based parameters
(even with privacy).

\begin{definition}
  \label{definition:bouned-moments-ratio}
  A random vector $z \in \R^d$ has \emph{moment ratio $r < \infty$}
  if for all $1 \le p \le 2 \log d$ and $1 \le j \le d$ 
  \begin{equation*}
    \E [ z_j^p ]^{2/p} \le r^2 p \cdot \E[ z_j^2 ].
  \end{equation*}
\end{definition}

When $z$ satisfies Def.~\ref{definition:bouned-moments-ratio}, we can
provide a
private procedure (Algorithm~\ref{alg:unknown-cov}) that provides good
approximation to the second moment of coordinates of $z_j$---and hence
higher-order moments---allowing the application of
a minimax optimal \PAGAN algorithm. We defer the proof
to~\Cref{sec:proof-unknown-cov}.
\begin{theorem}
  \label{thm:unknown-cov}
  Let $z$ have moment ratio $r$
  (Def.~\ref{definition:bouned-moments-ratio}) and let $\sigma_j^2 =
  \E[z_j^2]$.  Let $\beta>0$, $T = \frac{3}{2} \log d$,
  \begin{equation*}
    n \ge 1000 r^2
    \log\frac{8d}{\beta}
    \max \left\{\frac{T \sqrt{d} \log^2 r \log \frac{T}{\delta}}{\diffp}, r^2
    \right\},
  \end{equation*}
  and $\max_{1 \le j \le d} \sigma_j = 1$.  Then
  Algorithm~\ref{alg:unknown-cov} is $(\diffp,\delta)$-DP and
  outputs $\hat \sigma$ such that with probability $1 - \beta$,
  \begin{equation}
    \label{eqn:sigma-hat-good}
    \frac{1}{2} \max \{\sigma_j, d^{-3/2}\}
    \le \hat \sigma_j 
    \le 2 \sigma_j
    ~~~~ \mbox{for~all~} j \in [d].
  \end{equation}
  Moreover, when condition~\eqref{eqn:sigma-hat-good} holds,
  \PAGAN (Alg.~\ref{Algorithm2})
  with $\hat C_j = (r \hat \sigma_j)^{-4/3} / 4$, $p =\log
  d$ and $\batch = 1$ has convergence
  \begin{align*}
    & \E [f(\wb{x}^T;\Ds) - \min_{x \in \xdomain} f(x;\Ds)]
      \le \adaregret + \iftoggle{arxiv}{}{\\
    & \qquad~~~} \bigO(1) \diam_\infty(\xdomain)  r 
    \frac{\left(\sum_{j=1}^d \sigma_j ^{2/3}\right)^{3/2} }{n \diffp}
    \log \frac{d}{\delta}.
  \end{align*}
\end{theorem}

\begin{algorithm}[tb]
	\caption{Private Second Moment Estimation}
	\label{alg:unknown-cov}
	\begin{algorithmic}[1]
		\REQUIRE Dataset $\mathcal{S} = (\ds_1,\dots,\ds_n) \in \domain^n$, 
        number of iterations $T$, 
        privacy parameters $\diffp, \delta$;
		\STATE Set $\Delta = 1$	and $S = [d]$
        \FOR{$t=1$ to $T$\,}
            \FOR{$j \in S$\,}
                \STATE $ \rho_t \gets4 r \Delta \log r $
        	    \STATE $\hat \sigma_{t,j}^2 \gets \frac{1}{n} \sum_{i=1}^n \min(z_{i,j}^2,\rho_t^2) + \noise_{t,j}$ where $\noise_{t,j} \sim \normal(0,\rho_t^4 T^2 d \log(T/\delta)/n^2 \diffp^2)$
        	    \IF{$\hat \sigma_{t,j} \ge 2^{-t - 1}$}
                    \STATE $\hat \sigma_j \gets 2^{-t}$
                    \STATE $S \gets S \setminus \{j\}$
                \ENDIF    
        	\ENDFOR
        	\STATE $\Delta \gets  \Delta/2$
        \ENDFOR
        \FOR{$j \in S$\,}
             \STATE $\hat \sigma_j \gets 2^{-T}$
        \ENDFOR
        \STATE {\bfseries Return:} $(\hat \sigma_1,\dots,\hat \sigma_d)$
	\end{algorithmic}
\end{algorithm}

\section{Lower bounds for private optimization}
\label{sec:LB}

To give a more complete picture of the complexity of private stochastic
optimization, we now establish (nearly) sharp lower bounds, which in turn
establish the minimax optimality of \PAGAN and \PASAN. We establish this in
two parts, reflecting the necessary dependence on geometry in the
problems~\citep{LevyDu19}: in Section~\ref{sec:LB-ell-box}, we show that
\PAGAN achieves optimal complexity for minimization over $\xdomain_\infty =
\{x \in \R^d: \linf{x} \le 1 \}$. Moreover, in
Section~\ref{sec:LB-ell2-ball} we show that \PASAN achieves optimal rates in
the Euclidean case, that is, for domain $\xdomain_2 = \{x \in \R^d: \ltwo{x}
\le 1 \}$.

As one of our foci here is for data with varying norms, we prove lower bounds
for
sub-Gaussian data---the strongest setting for our upper bounds.
In particular, we shall consider linear functionals $F(x; z) = z^T x$,
where the entries $z_j$ of $z$ satisfy $|z_j| \le \sigma_j$ for a prescribed
$\sigma_j$; this is sufficient for the data $Z$ to
be $\frac{\sigma_j^2}{4}$-sub-Gaussian~\citep{Vershynin19}.
Moreover, our upper bounds are conditional on the observed sample
$\Ds$, and so we focus on this setting in our lower bounds,
where $|g_j| \le \sigma_j$ for all subgradients $g \in \partial F(x; z)$
and $j \in [d]$.

\subsection{Lower bounds for $\ell_\infty$-box}
\label{sec:LB-ell-box}

The starting point for our lower bounds for stochastic optimization over
$\xdomain_\infty$ is the following lower bound for the
problem of estimating the sign of the mean of a dataset. This will then
imply our main lower bound for private optimization. We defer the proof of
this result to Appendix~\ref{sec:proof-LB-signs}.

\begin{proposition}
  \label{prop:sign-LB-var}
  Let $\mech$ be $(\diffp,\delta)$-DP and $\Ds = (\ds_1,\dots,\ds_n) $ 
  where $\ds_i \in \domain = \{ z \in \R^d: |z_{j}| \le \sigma_j \}$. 
  Let $\bar z = \frac{1}{n} \sum_{i=1}^n z_i$ be the mean of the dataset $\Ds$.
  If $\sqrt{d} \log d \le n \diffp$, then
  \begin{align*}
    \sup_{\Ds \in \domain^n } 
    \E & \left[ \sum_{j=1}^d |\bar{\ds}_j| 
      \indic {\sign(\mech_j(\Ds)) \neq \sign(\bar{\ds}_j)} \right] \iftoggle{arxiv}{}{\\
    &} \ge \frac{(\sum_{j=1}^d \sigma_j^{2/3} )^{3/2} }{n \diffp \log^{5/2} d} .
  \end{align*} 
\end{proposition}

We can now use this lower bound to establish
a lower bound for private optimization over the 
$\ell_\infty$-box by an essentially straightforward reduction.
Consider the problem
\begin{equation*}
  \minimize_{x \in \xdomain_\infty} 
  f(x;\Ds) \defeq  - \frac{1}{n} \sum_{i=1}^n x^T \ds_i 
  = - x^T \bar{\ds},
\end{equation*}
where $\bar z = \frac{1}{n} \sum_{i=1}^n z_i$ is the mean of the dataset.
Letting $x^\star_\Ds \in \argmin_{x \in \xdomain_\infty} f(x; \Ds)$,
we have the following result.
\begin{theorem}
  \label{thm:LB-opt-l1}
  Let $\mech$ be $(\diffp,\delta)$-DP and $\Ds \in \domain^n$, where
  $\domain = \{ z \in \R^d: |z_{j}| \le \sigma_j \}$. If $\sqrt{d}
  \log d \le n \diffp$,
  then
  \begin{equation*}
    \sup_{\Ds \in \domain^n } 
    \E \left[ f(\mech(\Ds);\Ds) -  f(x^\star_\Ds;\Ds) \right] 
    \ge \frac{(\sum_{j=1}^d \sigma_j^{2/3})^{3/2}}{n \diffp \log^{5/2} d}.
  \end{equation*}
\end{theorem}
\begin{proof}
  For a given dataset $\Ds$,
  the minimizer $x\opt_j = \sign(\bar{\ds}_j)$.
  Therefore for every $x$ we have
  \begin{align*}
    f(x;\Ds) - f(x\opt;\Ds)
    & = \lone{\bar{\ds}} -  x^T \bar{\ds} \\
    & \ge \sum_{j=1}^d |\bar{\ds}_j| 
    \indic {\sign(x_j) \neq \sign(\bar{\ds}_j)} .
  \end{align*}
  As $\sign(\mech(\Ds))$ is $(\diffp,\delta)$-DP by post-processing, the
  claim follows from~\Cref{prop:sign-LB-var} by taking expectations.
\end{proof}

Recalling the upper bounds that \PAGAN achieves in~\Cref{sec:ub-adagrad},
\Cref{thm:LB-opt-l1} establishes the tightness of these bounds to
within logarithmic factors.

\subsection{Lower bounds for $\ell_2$-ball}
\label{sec:LB-ell2-ball}
Having established \PAGAN's optimality for $\ell_\infty$-box constraints,
in this section we turn to proving lower bounds for optimization over the
$\ell_2$-ball, which demostrate the optimality of \PASAN.  The lower
bound builds on the lower bounds of~\citet{BassilySmTh14}. Following their
arguments, let
$\xdomain_2 = \{x \in \R^d: \ltwo{x} \le 1 \}$ and consider
the problem
\begin{equation*}
  \minimize_{x \in \xdomain_2} 
  f(x;\Ds) \defeq  - \frac{1}{n} \sum_{i=1}^n x^T \ds_i 
  = - x^T \bar{\ds}.
\end{equation*}
The following bound follows by appropriate re-scaling of the data points in
Theorem 5.3 in~\cite{BassilySmTh14}
\begin{proposition}
 \label{prop:LB-opt-l2-identity}
 		Let $\mech$ be $(\diffp,\delta)$-DP and $\Ds = (\ds_1,\dots,\ds_n) $ 
	where $\ds_i \in \domain = \{ z \in \R^d: \linf{\ds} \le \sigma_j \}$.
 	Then
	\begin{equation*}
 	\sup_{\Ds \in \domain^n } 
 		\E \left[ f(\mech(\Ds);\Ds) -  f(x^\star_\Ds;\Ds) \right] 
 			\ge \min \left( \sigma \sqrt{d} , \frac{d \sigma}{n \diffp} \right).
 	\end{equation*}
 \end{proposition}

Using~\Cref{prop:LB-opt-l2-identity}, we can establish the 
tight lower bounds---to within logarithmic factors---for
\PASAN (Section~\ref{sec:algs}).
We defer the proof to Appendix~\ref{sec:proof-LB-l2}.
\begin{theorem}
  \label{thm:LB-opt-l2-var}
 	Let $\mech$ be $(\diffp,\delta)$-DP and $\Ds = (\ds_1,\dots,\ds_n) $ 
	where $\ds_i \in \domain = \{ z \in \R^d: |z_{j}| \le \sigma_j \}$. 
 	If $\sqrt{d} \le n \diffp$, then
	\begin{equation*}
 	\sup_{\Ds \in \domain^n } 
 		\E \left[ f(\mech(\Ds);\Ds) -  f(x^\star_\Ds;\Ds) \right] 
 			\ge  \frac{\sum_{j=1}^d \sigma_j }{n \diffp \log d}.
 	\end{equation*}
 \end{theorem}

\section{Experiments}
\label{sec:experiments}

We conclude the paper with several experiments to demonstrate the
performance of \PAGAN~and \PASAN~algorithms. We perform experiments
both on synthetic data, where we
may control all aspects of the experiment, and a real-world example
training large-scale private language models.

\subsection{Regression on Synthetic Datasets}
\label{sec:exp-syn}

\iftoggle{arxiv}{
\begin{figure}[t]
  \begin{center}
    \begin{tabular}{ccc}
      \begin{overpic}[width=.34\columnwidth]{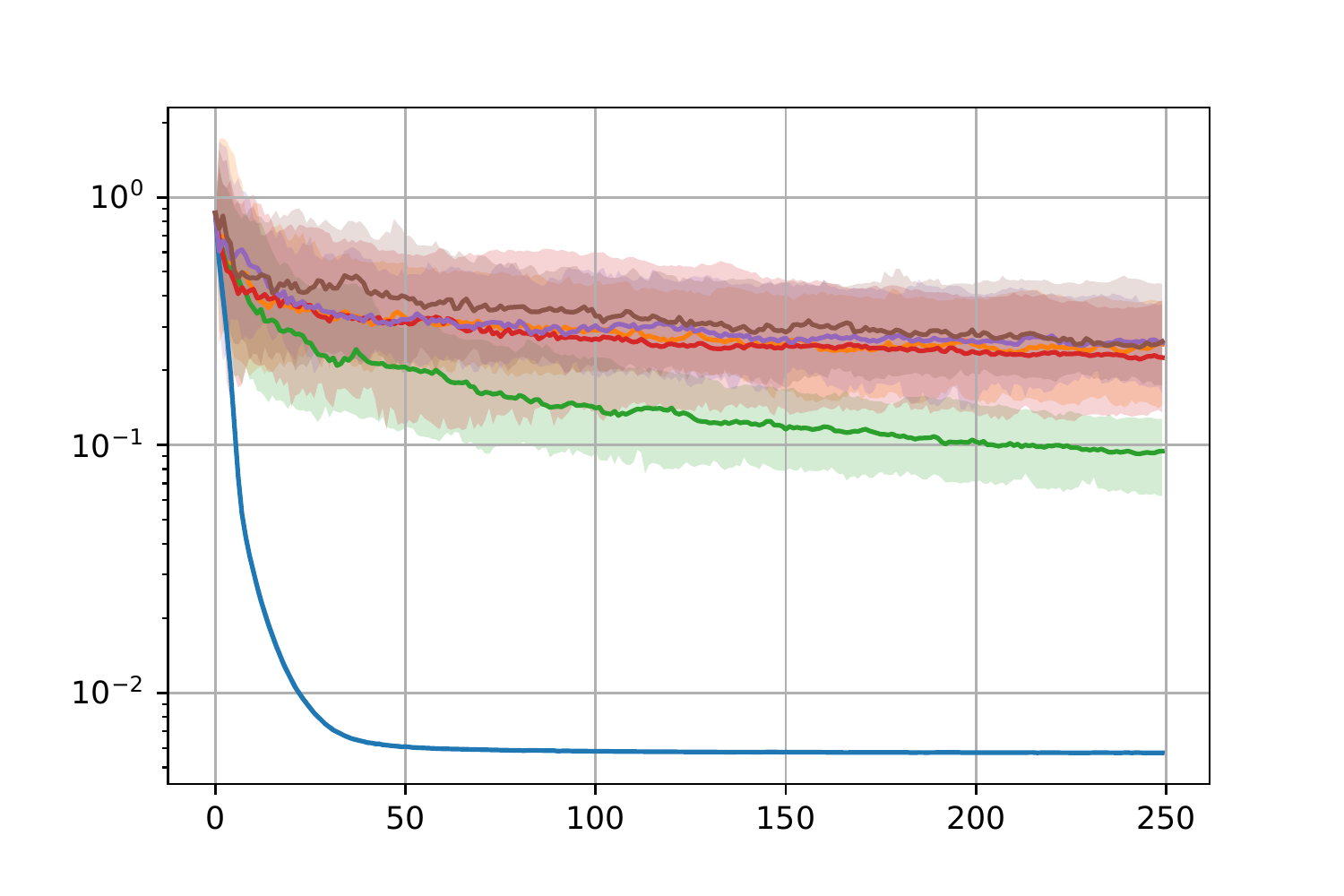}
      \end{overpic} &
      \hspace{-1cm}
      \begin{overpic}[width=.34\columnwidth]{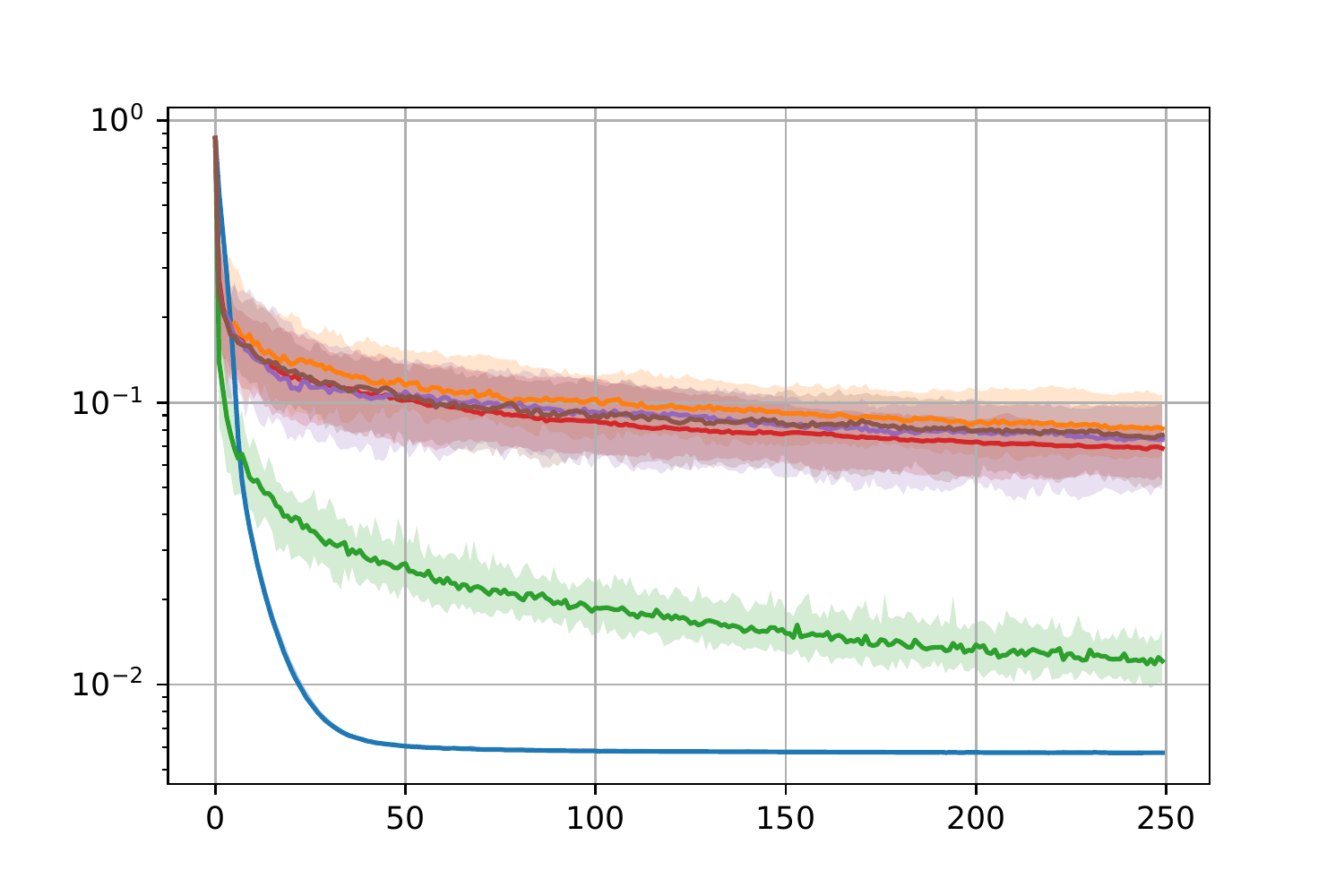}
      \end{overpic} &
      \hspace{-1cm}
      \begin{overpic}[width=.34\columnwidth]{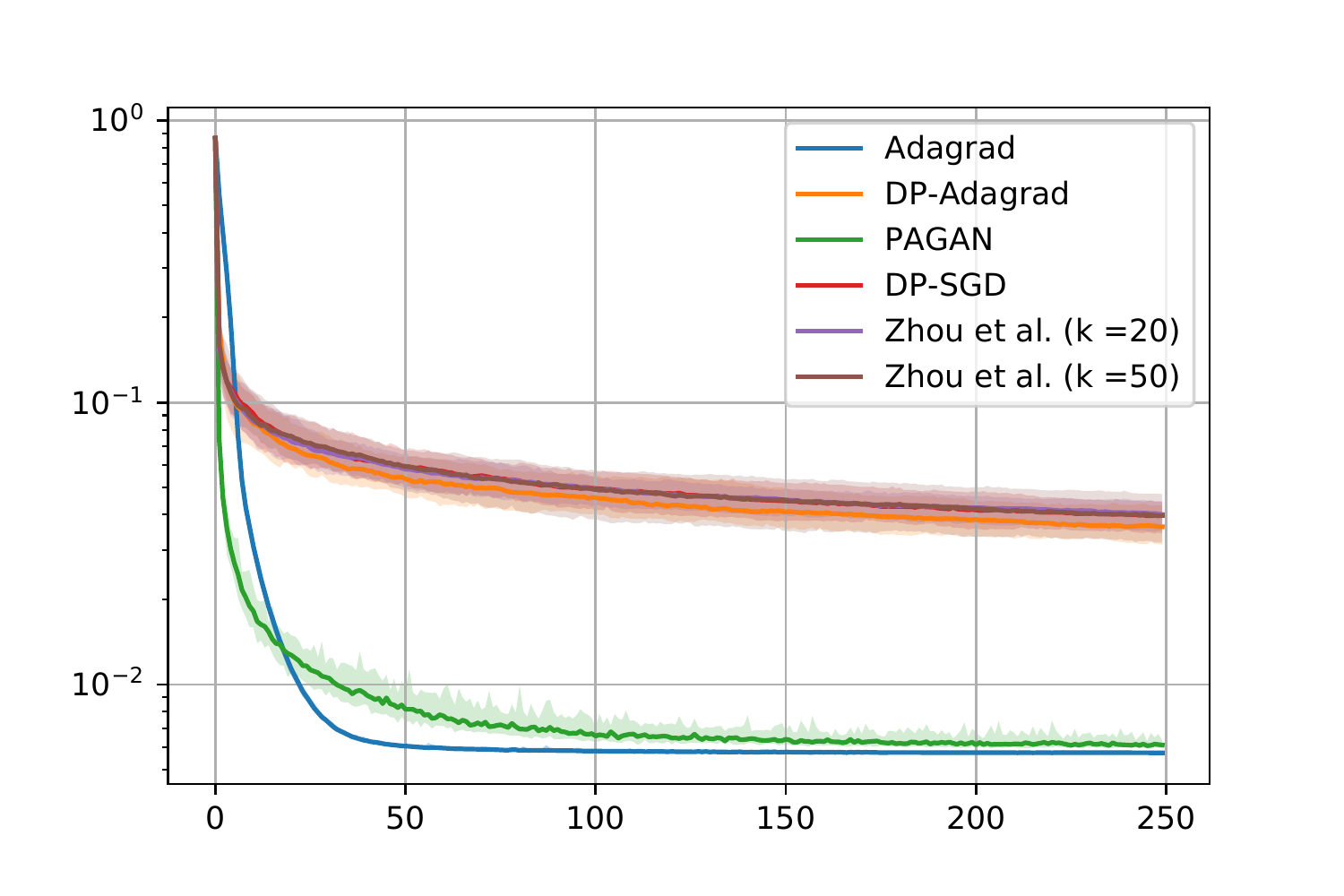}
      \end{overpic}
      \\
      (a) & (b)  & (c)
    \end{tabular}
    \caption{\label{fig:abs-reg} Sample loss as a function of the iterate for
      various optimization methods for synthetic
      absolute regression problem~\eqref{eqn:abs-regression}
      with varying privacy parameters $\diffp$. (a)
      $\diffp = 0.1$. (b) $\diffp = 1$. (c) $\diffp = 4$. }
  \end{center}
\end{figure}
}
{
\begin{figure*}[t]
  \begin{center}
    \begin{tabular}{ccc}
      \begin{overpic}[width=.75\columnwidth]{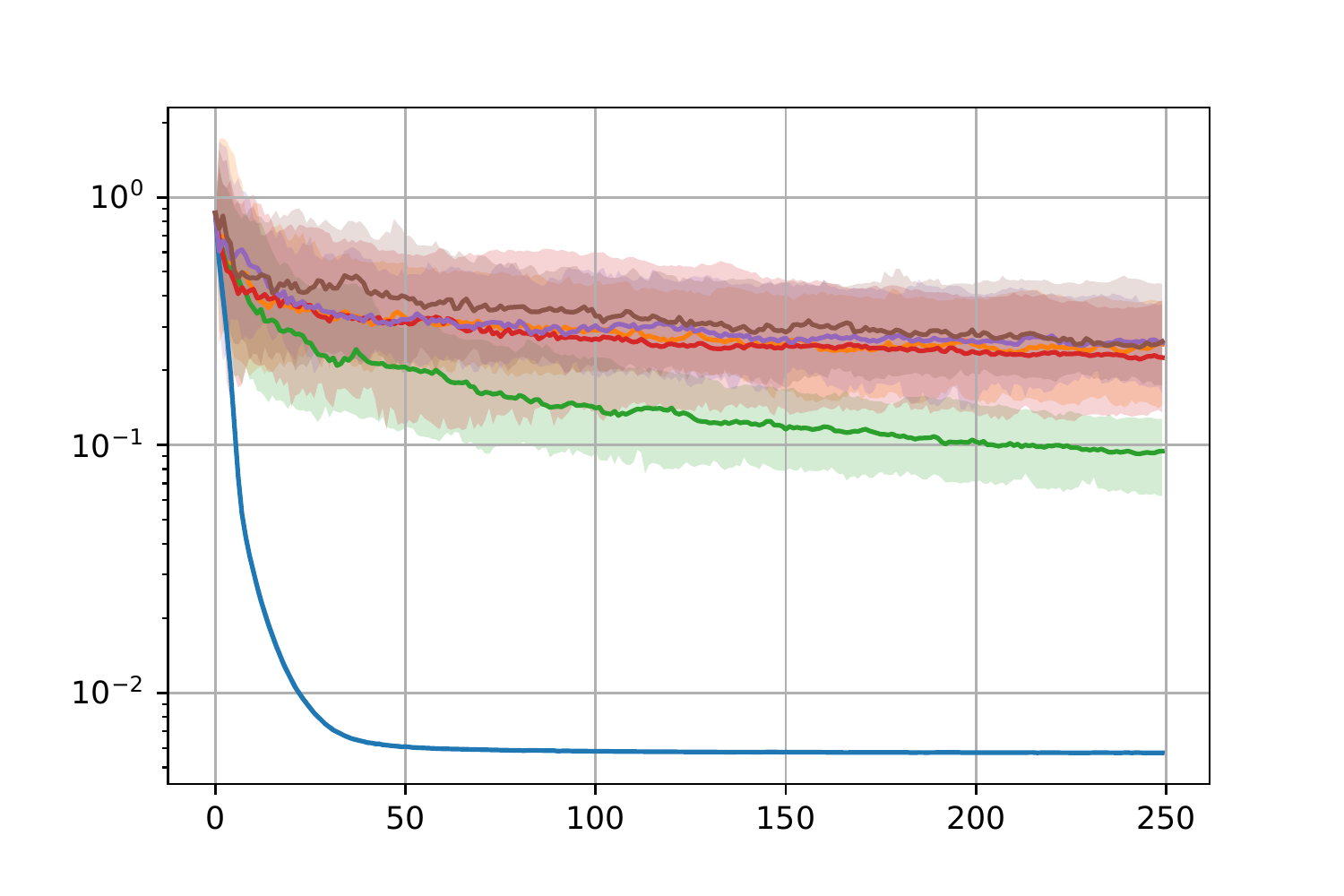}
      \end{overpic} &
      \hspace{-1cm}
      \begin{overpic}[width=.75\columnwidth]{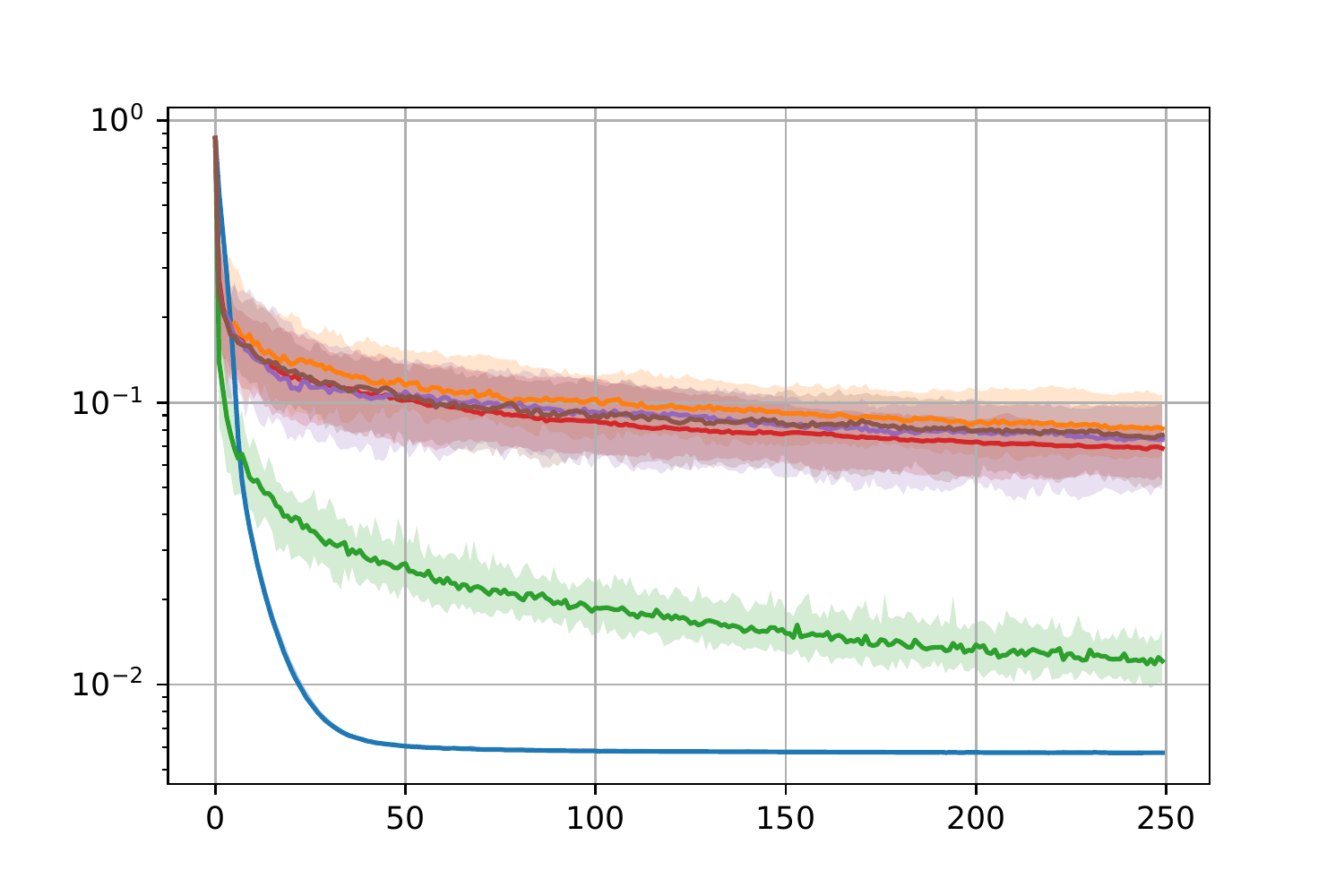}

      \end{overpic} &
      \hspace{-1cm}
      \begin{overpic}[width=.75\columnwidth]{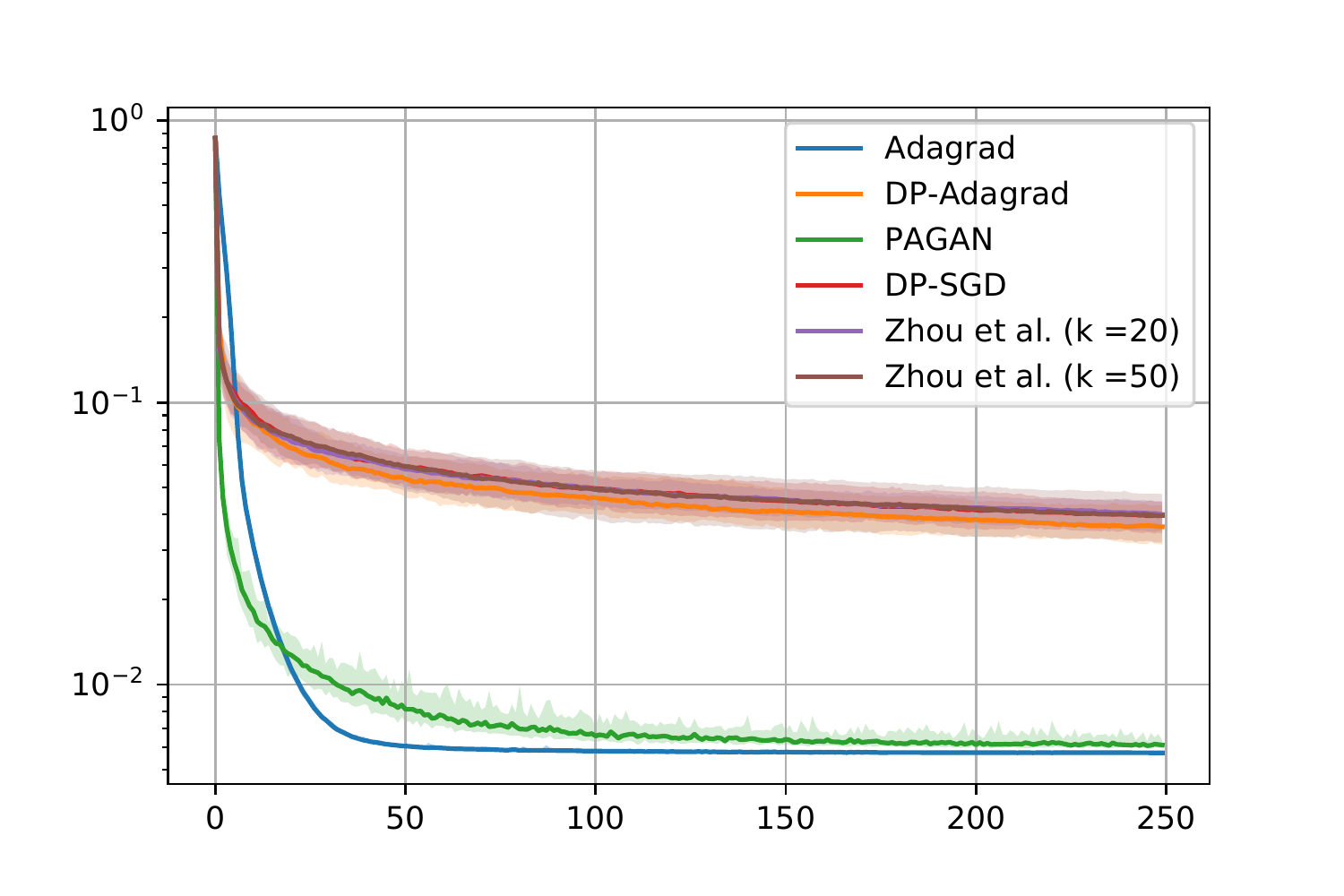}
      \end{overpic}
      \\
      (a) & (b)  & (c)
    \end{tabular}
    \caption{\label{fig:abs-reg} Sample loss as a function of the iterate for
      various optimization methods for synthetic
      absolute regression problem~\eqref{eqn:abs-regression}
      with varying privacy parameters $\diffp$. (a)
      $\diffp = 0.1$. (b) $\diffp = 1$. (c) $\diffp = 4$. }
  \end{center}
\end{figure*}
}

If our \PAGAN~algorithm indeed captures the aspects of AdaGrad and other
adaptive methods, we expect it to outperform other private stochastic
optimization methods at the least in those scenarios where AdaGrad improves
upon stochastic gradient methods---as basic sanity check.  To that end, in
our first collection of experiments, we compare \PAGAN against standard
implementations of private AdaGrad and SGD methods. We also compare our
method against Projected DP-SGD (PDP-SGD)~\cite{ZhouWuBa20}, which projects
the noisy gradients into the (low-dimensional) subspace of the top $k$
eigenvectors of the second moment of gradients.

We consider an absolute regression problem with data $(a_i, b_i)
\in \R^d \times \R$ and loss $F(x;a_i,b_i) = |\<a_i,x\> -
b_i|$. Given $n$ datapoints $(a_1,b_1),\dots,(a_n,b_n)$, we
wish to solve
\begin{equation}
  \label{eqn:abs-regression}
  \mbox{minimize} ~~
  f(x) = \frac{1}{n} \sum_{i=1}^n |\<a_i,x\> - b_i|.   
\end{equation}
We construct the data by drawing
an optimal $x\opt \sim \uniform\{-1, 1\}^d$,
sampling $a_i \simiid \normal(0,
\diag(\sigma)^2)$ for a vector $\sigma \in \R_+^d$,
and setting $b_i = \<a_i,x\opt\> + \noise_i$ for noise $\noise_i
\simiid \laplace(0,\tau)$, where $\tau \ge 0$.

We compare several algorithms in this experiment: non-private AdaGrad; the
naive implementations of private SGD (\PASAN, Alg.~\ref{Algorithm1}) and
AdaGrad (\PAGAN, Alg.~\ref{Algorithm2}), with $A_k = I$; \PAGAN~with the
optimal diagonal matrix scaling $A_k$ we derive in~\Cref{sec:ub-adagrad};
and Zhou et al.'s PDP-SGD with ranks $k = 20$ and $k = 50$. In our
experiments, we use the parameters $n = 5000$, $d=100$, $\sigma_j =
j^{-3/2}$, $\tau = 0.01$, and the batch size for all methods is $b = 70$.
As optimization methods are sensitive
to stepsize choice even non-privately~\cite{AsiDu19siopt}, we
run each method with different values of initial stepsize in $\{ 0.005,
0.01, 0.05, 0.1, 0.15, 0.2, 0.4, 0.5, 1.0 \}$ to find the best stepsize value. Then we run each method $T=30$ times and report the median of the loss as a function of the iterate with 95\% confidence intervals.

\Cref{fig:abs-reg} demonstrates the results of this experiment.  Each plot
shows the loss of the methods against iteration count in various privacy
regimes. In the high-privacy setting (\Cref{fig:abs-reg}(a)), the
performance of all private methods is worse than the non-private algorithms,
though \PAGAN~(Alg.~\ref{Algorithm2}) seems to be outperforming other
algorithms. As we increase the privacy parameter---reducing privacy
preserved---we see that \PAGAN~quickly starts to enjoy faster convergence,
resembling non-private AdaGrad. 
different for non-private methods).  In contrast, the standard
implementation of private AdaGrad---even in the moderate privacy regime with
$\diffp=4$---appears to obtain the slower convergence of SGD rather than the
adaptive methods. This is consistent with the predictions our theory makes:
the isometric Gaussian noise addition that standard private stochastic
gradient methods (e.g.\ \PASAN~and variants) employ eliminates the geometric
properties of gradients (e.g., sparsity) that adaptive methods can---indeed,
must~\citep{LevyDu19}---leverage for improved convergence.

\subsection{Training Private Language Models on WikiText-2}
\label{sec:exp-lm}

Following our simulation results, we study the performance of \PAGAN and \PASAN for fitting a next word prediction model. Here, we train a variant of
a recurrent neural network with Long-Short-Term-Memory
(LSTM)~\cite{HochreiterJu97} on the WikiText-2
dataset~\cite{MerityXiBrSo17}, which is split into train, validation, and
test sets. We further split the train set to 59,674 data points, where each
data point has $35$ tokens. The input data to the model consists of a
one-hot-vector $x \in \{0, 1\}^{d}$, where $d = 8,\!000$. The first
7,999 coordinates correspond to the most frequent tokens in the
training set, and the model reserves the last coordinate for unknown/unseen
tokens. We train a full network, which consists of a fully connected
embedding layer $x \mapsto Wx$ mapping to 120 dimensions, where $W \in
\R^{120 \times 8000}$; two layers of LSTM units with 120 hidden units,
which then output a vector $h \in \R^{120}$; followed by a fully
connected layer $h \mapsto \Theta h$ expanding the representation
via $\Theta \in \R^{8000 \times 120}$; and then into a softmax
layer to emit next-word probabilities via a logistic regression model.
The entire model contains $2,\!160,\!320$ trainable parameters. 

We use Abadi et al.'s moments accountant analysis~\cite{AbadiChGoMcMiTaZh16} to
track the privacy losses of each of the methods.  In each experiment, for
\PAGAN and \PASAN we use gradients trained for one epoch on a held-out dataset (a subset of the WikiText 103 dataset~\cite{MerityXiBrSo17} which does not intersect with WikiText-2) to estimate moment
bounds and gradient norms, as in Section~\ref{sec:unknown-cov}; these
choices---while not private---reflect the common practice that we may have
access to public data that provides a reasonable proxy for the actual
moments on our data. Moreover, our convergence guarantees in
Section~\ref{sec:algs} are robust in the typical sense of stochastic
gradient methods~\citep{NemirovskiJuLaSh09}, in that mis-specifying the
moments by a multiplicative constant factor yields only constant factor
degradation in convergence rate guarantees, so we view this as an acceptable
tradeoff in practice.
It is worth noting that we ignore the model trained over the public data and use that one epoch solely for estimating the second moment of gradients.



In our experiments, we evaluate the performance of the trained models with
validation- and test-set perplexity. While we
propose adaptive algorithms, we still require
hyperparameter tuning, and thus perform a hyper-parameter search
over three algorithm-specific constants:
a multiplier $\alpha \in \{0.1, 0.2, 0.4, 0.8, 1.0, 10.0, 50.0 \}$ for step-size,
mini-batch size $b = 250$, and projection
threshold $B \in \{0.05, 0.1, 0.5, 1.0\}$. Each run of these algorithms takes $<$ 4 hours on 
a standard workstation without any accelerators.
We trained the LSTM model above with
\PAGAN and \PASAN and compare its performance with DP-SGD \cite{AbadiChGoMcMiTaZh16}.
We also include completely
non-private SGD and AdaGrad for reference. 
{ 
  We do not include PDP-SGD \cite{ZhouWuBa20} in this experiment as, for our
  $N=2.1 \cdot 10^6$ parameter model, computing the low-rank subspace for
  gradient projection that PDP-SGD requires is quite challenging.  Indeed,
  computing the gradient covariance matrix Zhou et al.~\cite{ZhouWuBa20}
  recommend is certainly infeasible. While power iteration or Oja's method
  can make computing a $k$-dimensional projection matrix feasible, the
  additional memory footprint of this $kN$-sized matrix (compared to the
  original model size $N$) can be prohibitive, restricting us to smaller
  models or very small $k$.
  For such small values ($k=50$),
  our experiments show that PDP-SGD achieves significantly worse error than the algorithms we consider hence we do not include it in the plots.
  On the other hand, our (diagonal)
  approach, like diagonal AdaGrad, only requires an additional memory of
  size $N$.  }

For
each of the privacy levels
$\diffp \in \{0.5, 1, 3\}$ we consider, we
present the performance of each algorithm in terms of best validation set
and test-set perplexity in Figure~\ref{fig:perplexing} and Table \ref{Table1}. 

\iftoggle{arxiv}{
\begin{figure}[t]
  \begin{center}
    \begin{tabular}{ccc}
      \begin{overpic}[width=.32\columnwidth]{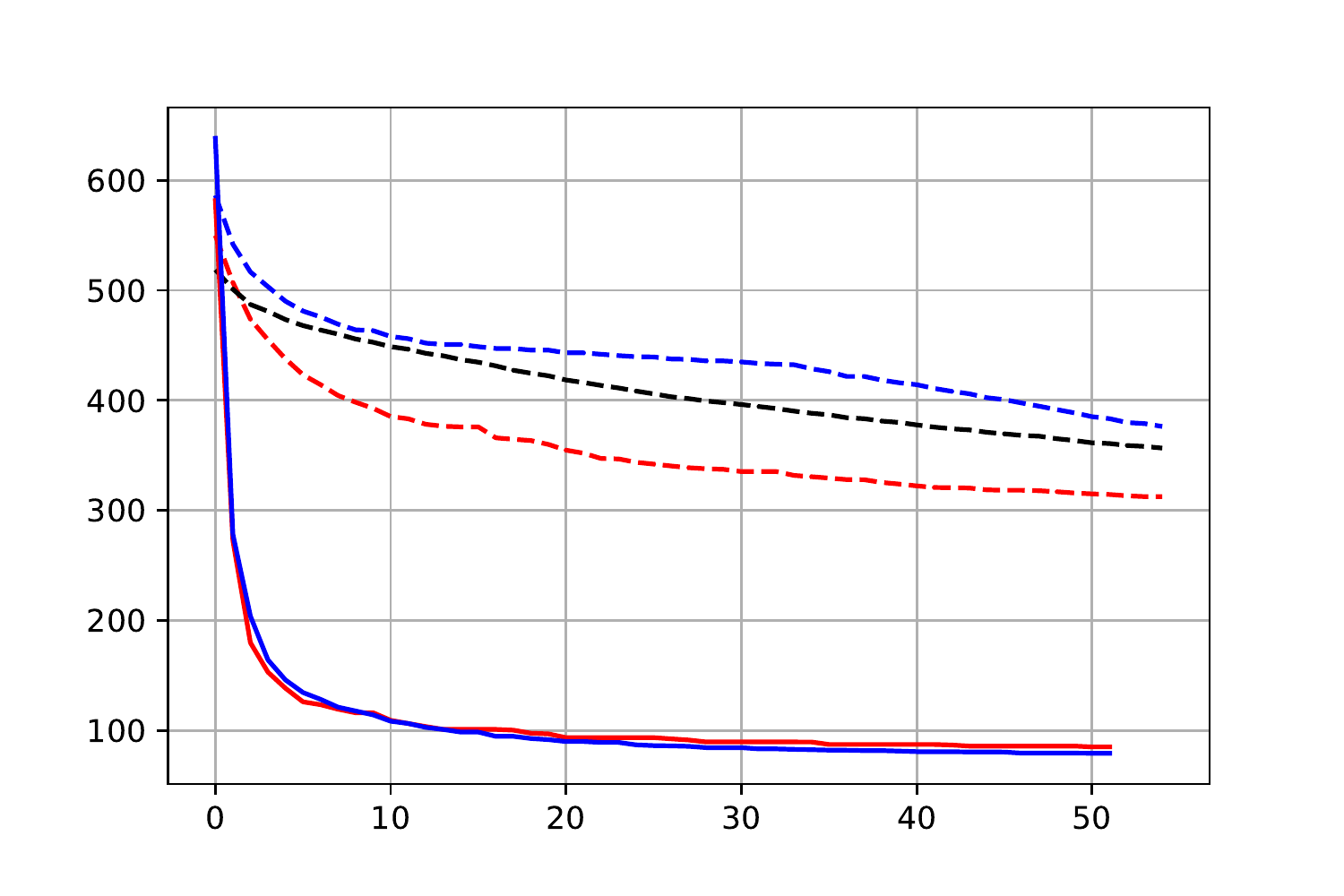}
      \end{overpic} &
      \begin{overpic}[width=.32\columnwidth]{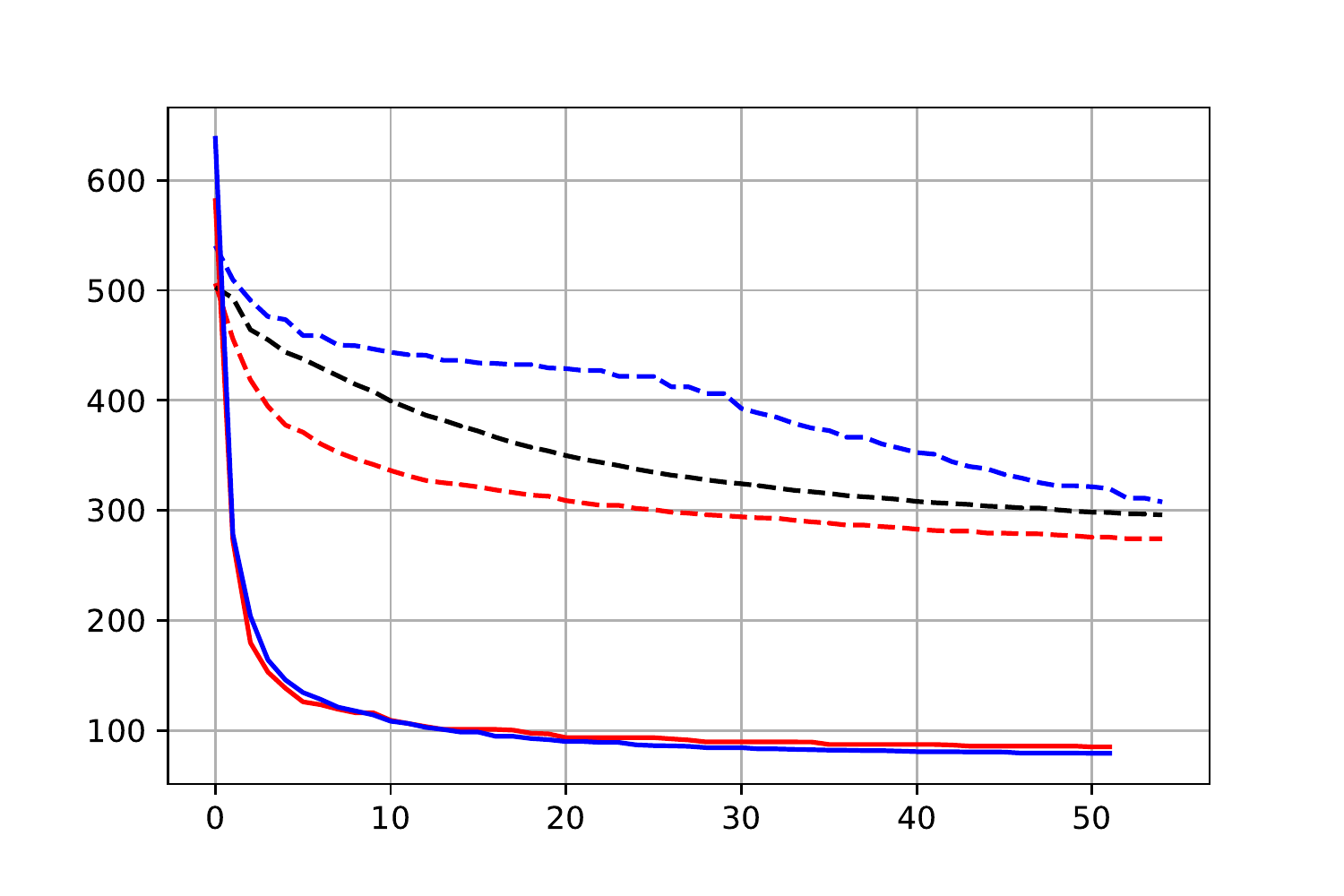}
      \end{overpic} &
      \begin{overpic}[width=.32\columnwidth]{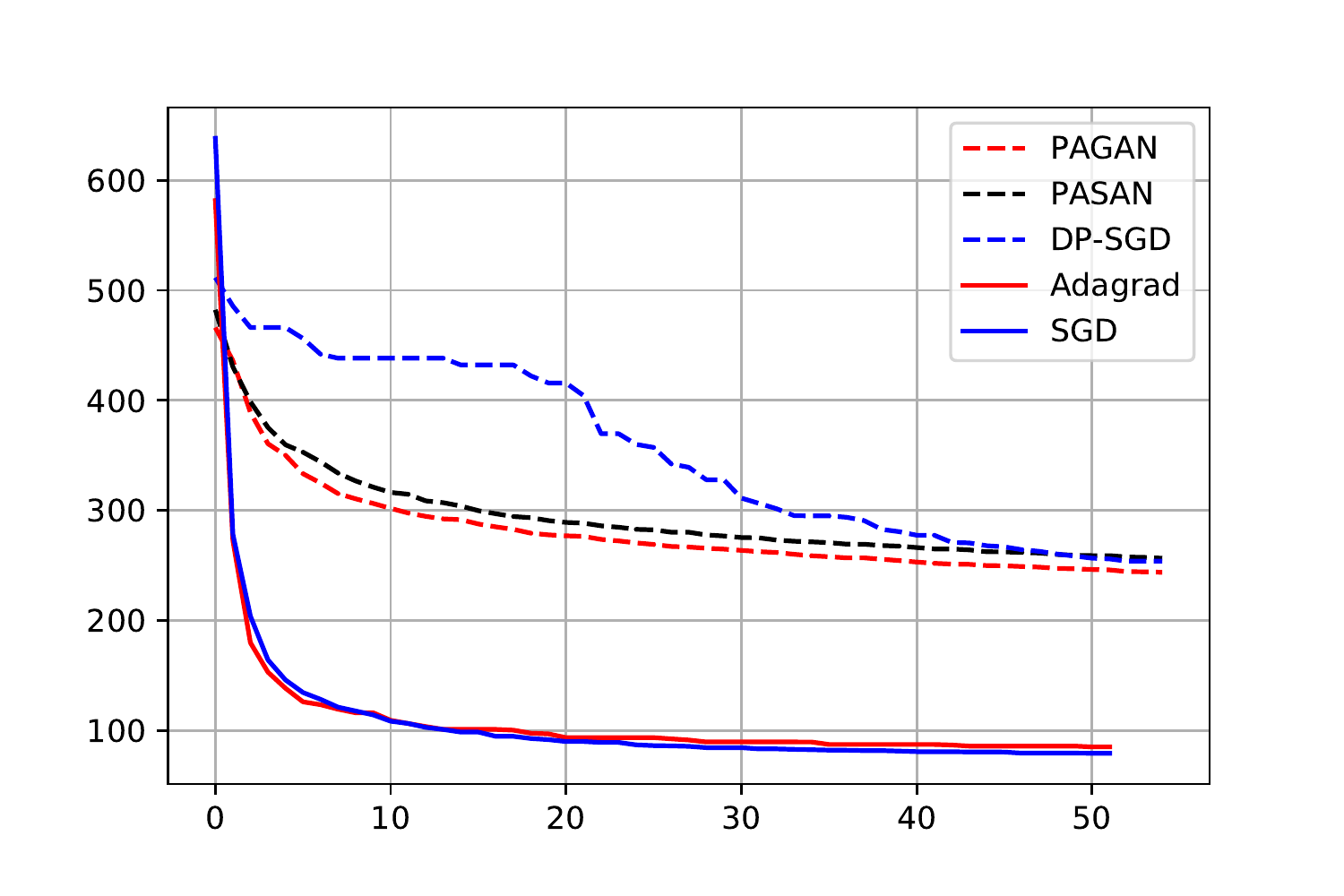}
      \end{overpic} 
      \\
      (a) & (b)  & (c)
    \end{tabular}
    \caption{\label{fig:perplexing} Minimum validation perplexity versus training rounds for seven epochs for \PAGAN, \PASAN, and the standard differentially private stochastic gradient method (DP{-}SGD) \cite{AbadiChGoMcMiTaZh16},
    varying privacy levels (a) $\diffp = .5$, (b) $\diffp = 1$ and (b) $\diffp = 3$}
  \end{center}
\end{figure}
}
{
\begin{figure*}[t]
  \begin{center}
    \begin{tabular}{ccc}
    \begin{overpic}[width=.7\columnwidth]{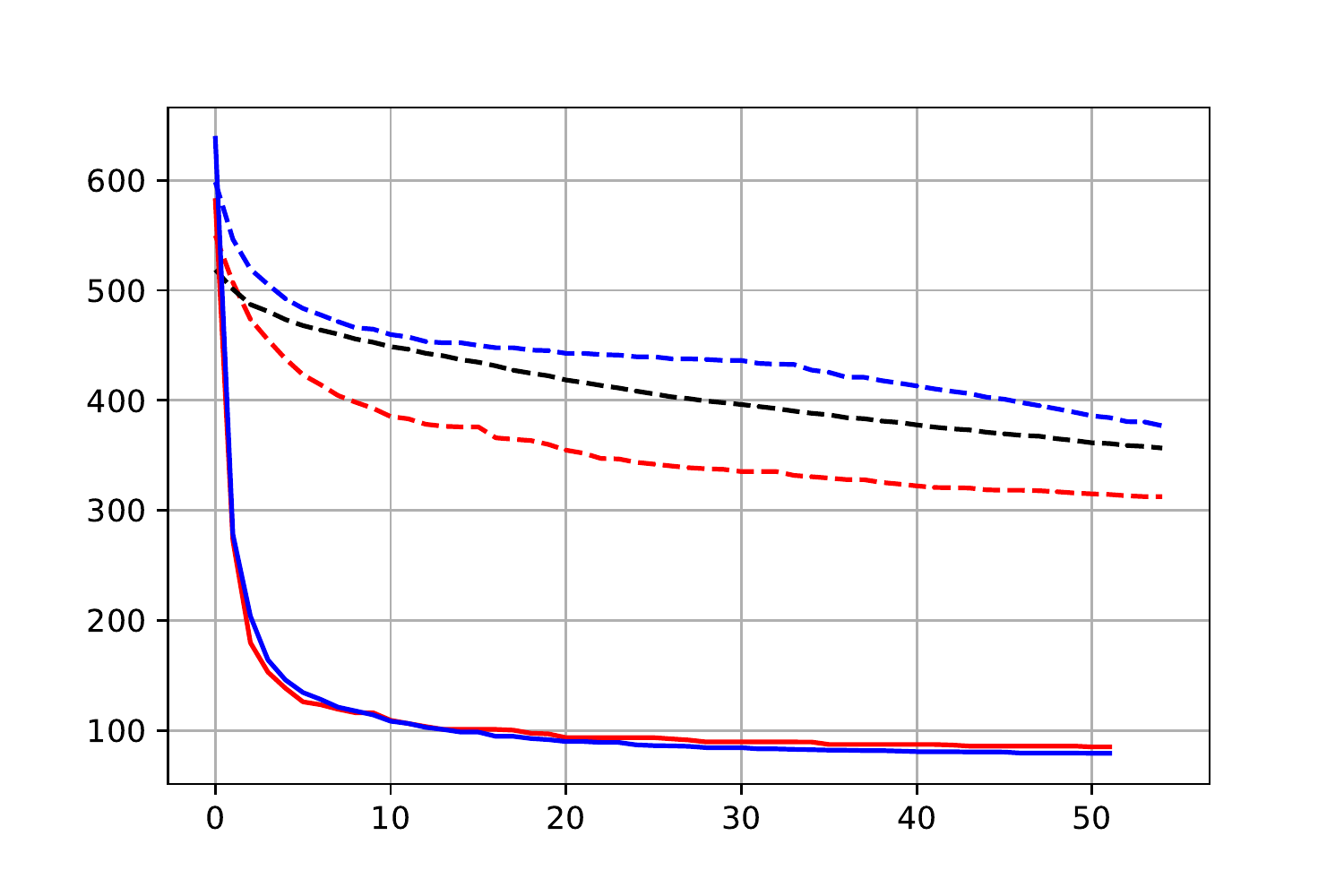}
      \end{overpic} &
      \hspace{-.3cm}
      \begin{overpic}[width=.7\columnwidth]{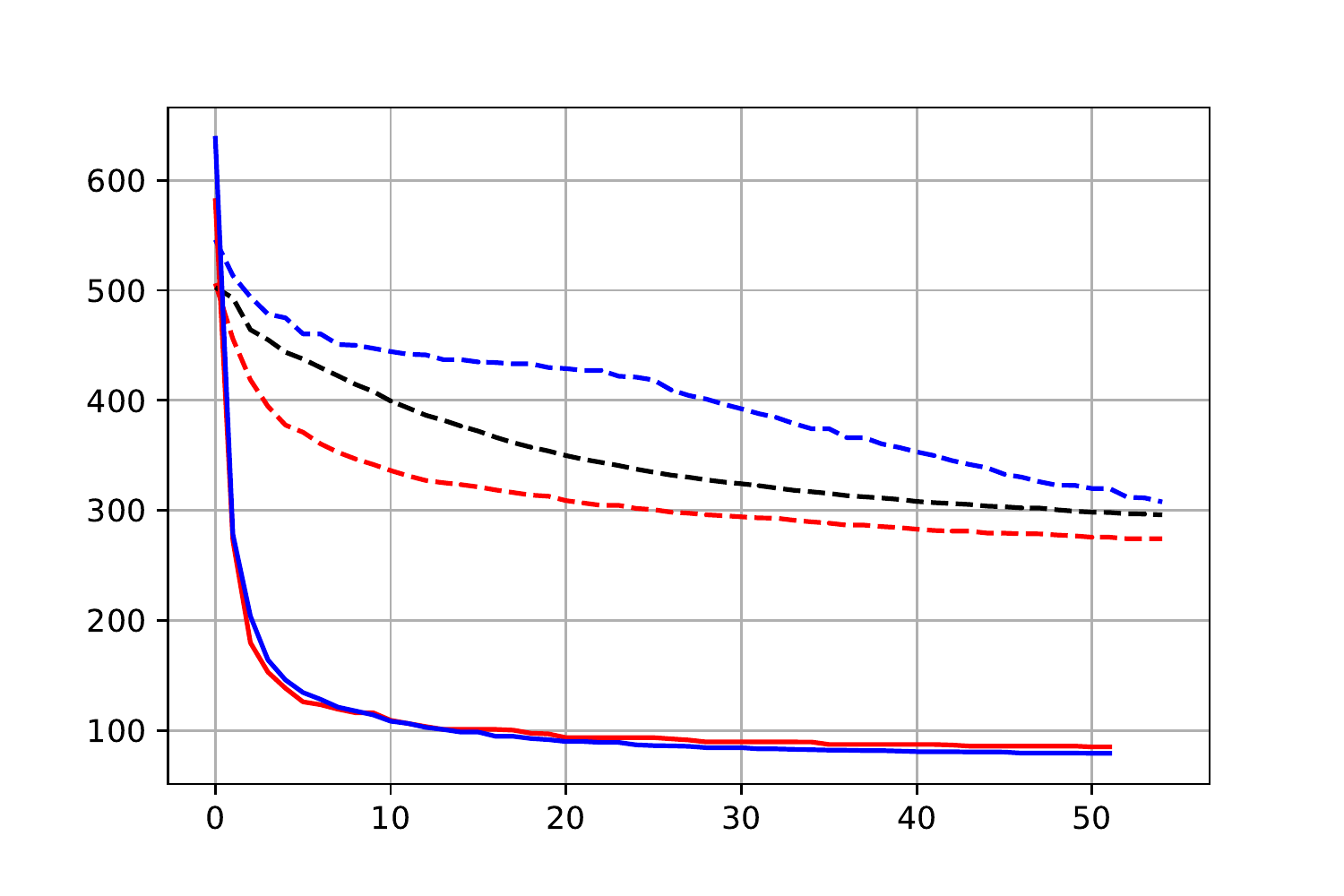}
      \end{overpic} &
      \hspace{-.3cm}
      \begin{overpic}[width=.7\columnwidth]{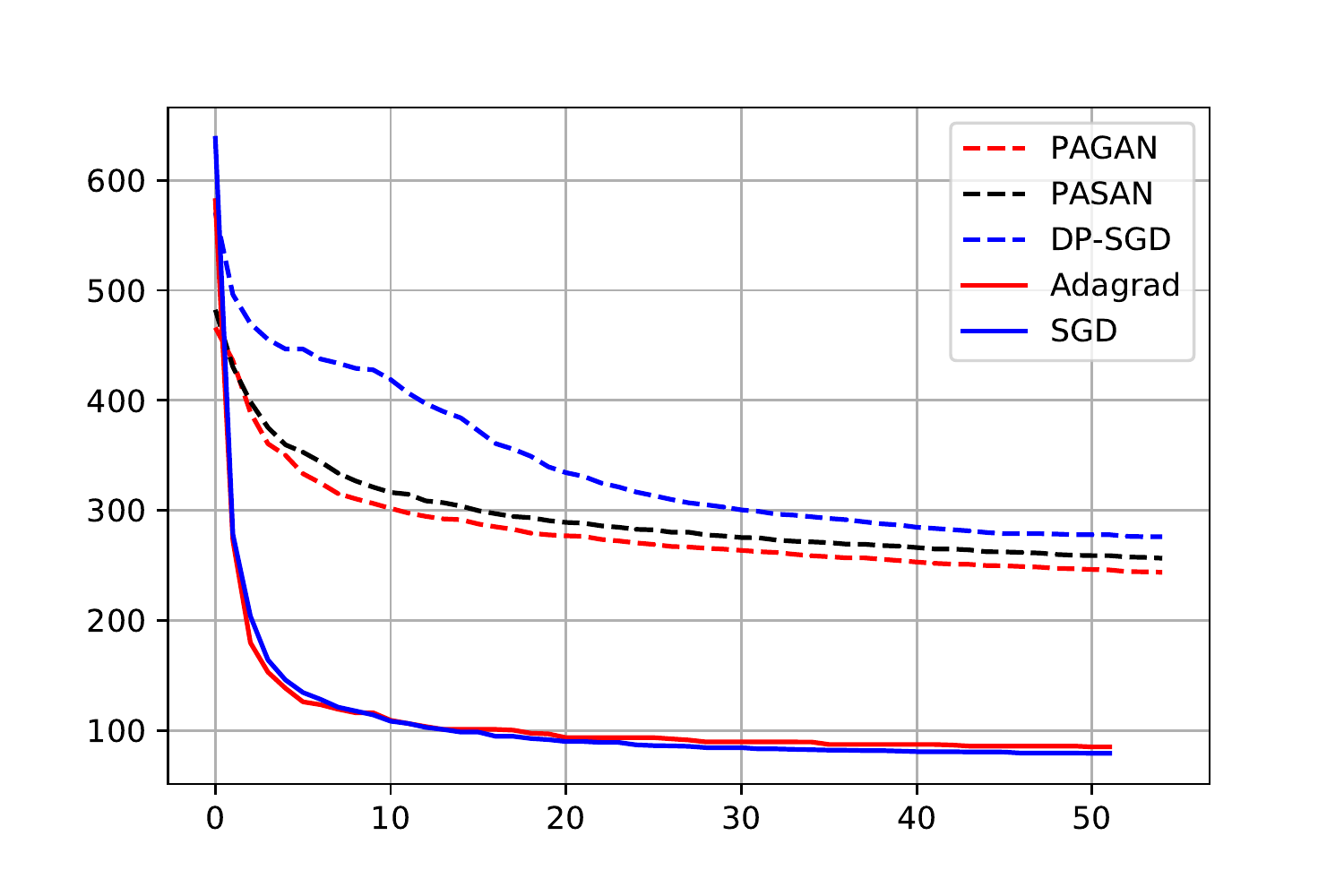}

      \end{overpic} 
      \\
      (a) & (b) & (c) 
    \end{tabular}
    \caption{\label{fig:perplexing} Minimum validation perplexity versus training rounds for seven epochs for \PAGAN, \PASAN, and the standard differentially private stochastic gradient method (DP{-}SGD) \cite{AbadiChGoMcMiTaZh16},
    varying privacy levels (a) $\diffp = .5$, (b) $\diffp = 1$ and (b) $\diffp = 3$}
  \end{center}
\end{figure*}
}

\newcommand{\bftab}{\fontseries{b}\selectfont}

\begin{table}[t]
\caption{Test perplexity error of different methods. For reference, non-private SGD and AdaGrad (without clipping) achieve 75.45 and 79.74, respectively. \label{Table1}}
\vspace{0.4cm}
\begin{center}
\begin{tabular}{|l|lll|}
\hline
Algorithm &  $\diffp=3$ & $\diffp=1$  & $\diffp=.5$ \\
\hline

DP-SGD \cite{AbadiChGoMcMiTaZh16}    & 238.44       & 285.11 & 350.23   \\ 

\PASAN            &  238.87     &  274.63 & 332.52     \\  

\PAGAN            & \bftab 224.82     & \bftab 253.41     & \bftab 291.41  \\ 
\hline
\end{tabular}
\end{center}
\end{table}

We highlight a few messages present in Figure~\ref{fig:perplexing}.  First,
\PAGAN consistently outperforms the non-adaptive methods---though all allow
the same hyperparameter tuning---at all privacy levels, excepting the
non-private $\diffp =
+\infty$, where \PAGAN without clipping is just AdaGrad
and its performance is comparable to the non-private stochastic gradient
method.  Certainly, there remain non-negligible gaps between the performance
of the private methods and non-private methods, but we hope that this
is a step at least toward effective large-scale private optimization and
modeling.

\section{Acknowledgement}
The authors like to thank Vitaly Feldman and Jalaj Upadhyay for helpful discussions in the process of preparing this paper and Daniel Levy for comments on an earlier draft.
Part of this work was done while HA and AF were interning at Apple.

\printbibliography

@string{acm = {Association for Computing Machinery}}

@string{colt10 = {Proceedings of the Twenty Third Annual Conference on
		  Computational Learning Theory}}

@string{colt13 = {Proceedings of the Twenty Sixth Annual Conference on
		  Computational Learning Theory}}

@string{focs13 = {54th Annual Symposium on Foundations of Computer Science}}

@string{focs14 = {55th Annual Symposium on Foundations of Computer Science}}

@string{stoc20 = {Proceedings of the Fifty-Second Annual ACM
		  Symposium on the Theory of Computing}}

@string{jasa =	{Journal of the American Statistical Association}}

@string{jmlr =	{Journal of Machine Learning Research}}

@string{mit =	{Massachusetts Institute of Technology}}

@string{nips20= {Advances in Neural Information Processing Systems 20}}

@string{nips2013= {Advances in Neural Information Processing Systems 26}}

@string{nips2015= {Advances in Neural Information Processing Systems 28}}

@string{nips2019= {Advances in Neural Information Processing Systems 32}}

@string{siopt={SIAM Journal on Optimization}}

@string{icml03 = {Proceedings of the Twentieth International Conference on Machine Learning}}

@string{iclr17 = {Proceedings of the Fifth International Conference on Learning Representations}}

@string{siopt = {SIAM Journal on Optimization}}

@inproceedings{AbadiChGoMcMiTaZh16,
title = {Deep Learning with Differential Privacy},
author  = {Martin Abadi and Andy Chu and Ian Goodfellow and Brendan McMahan and Ilya Mironov and Kunal Talwar and Li Zhang},
year  = {2016},
booktitle = {23rd ACM Conference on Computer and Communications Security (ACM CCS)},
pages = {308--318}
}

@article{AsiDu19siopt,
author = {Hilal Asi and John C. Duchi},
year = 2019,
title = {Stochastic (Approximate) Proximal Point Methods:
Convergence, Optimality, and Adaptivity},
journal = siopt,
volume = 29,
number = 3,
pages = {2257--2290},
url = {https://arXiv.org/abs/1810.05633},
}

@article{AsiFeKoTa21,
author = {Hilal Asi and Vitaly Feldman and Tomer Koren and Kunal Talwar},
title = {Private Stochastic Convex Optimization: Optimal Rates in {$\ell_1$} Geometry},
year = 2021,
journal = {arXiv:2103.01516 [cs.LG]},
}

@article{BarberDu14a,
author = {Rina Foygel Barber and John C. Duchi},
title = {Privacy and Statistical Risk: Formalisms and Minimax Bounds},
year = 2014,
journal = {arXiv:1412.4451 [math.ST]},
}

@inproceedings{BartlettHaRa07,
author = {P. L. Bartlett and E. Hazan and A. Rakhlin},
title = {Adaptive Online Gradient Descent},
year = 2007,
booktitle = nips20,
}

@inproceedings{BassilyFeTaTh19,
title = {Private Stochastic Convex Optimization with Optimal Rates},
author = {Raef Bassily and Vitaly Feldman and Kunal Talwar and Abhradeep Thakurta},
year = 2019,
booktitle = nips2019,
}

@inproceedings{BassilySmTh14,
  title={Private empirical risk minimization: {E}fficient algorithms and tight error bounds},
  author={Raef Bassily and Adam Smith and Abhradeep Thakurta},
  booktitle= focs14,
  pages={464--473},
  year= 2014,
}

@article{BeckTe03,
  author = 	 {A. Beck and M. Teboulle},
  title = 	 {Mirror Descent and Nonlinear Projected Subgradient Methods for Convex Optimization},
  journal = 	 {Operations Research Letters},
  year = 	 {2003},
  volume = 	 {31},
  pages = 	 {167--175}
}

@article{ChaudhuriMoSa11,
author = {Kamalika Chaudhuri and Claire Monteleoni and Anand D. Sarwate},
title = {Differentially private empirical risk minimization},
year = 2011,
journal = jmlr,
volume = 12,
pages = {1069--1109},
comment = {
Uses differential privacy as measure of privacy for SVM (and similar ERM
problems). Studies both output perturbation--adding noise to learned weight
vector--and objective perturbation, which involves adding a random linear
term to ERM objective. Theory and practical results in paper suggest
objective perturbation is better. Analysis follows from usual stability
of ERM, then additional noise to make sure things are appropriately
smooth (and densities exist). The objective perturbation technique requires
showing that (under suitable differentiability conditions) there is a bijection
between the output weight vector and perturbing linear term. Then
one can compute a change of variables and use that to talk about the
density of the final weight vector based on the distribution of the
perturbing linear term. Convergence rates are a bit odd, and
algorithm to do model selection while maintaining privacy destroys rates
of convergence.
},
}

@misc{Duchi19,
author = {John C. Duchi},
title = {Information Theory and Statistics},
year = 2019,
howpublished = {Lecture Notes for Statistics 311/{EE} 377,
Stanford University},
note = {Accessed May 2019},
url = {http://web.stanford.edu/class/stats311/lecture-notes.pdf},
}

@incollection{Duchi18,
author = {John C. Duchi},
title = {Introductory Lectures on Stochastic Convex Optimization},
booktitle = {The Mathematics of Data},
series = {IAS/Park City Mathematics Series},
publisher = {American Mathematical Society},
editors=  {Michael Mahoney and John C. Duchi and Anna Gilbert},
year = 2018,
}

@article{DuchiHaSi11,
author = {John C. Duchi and Elad Hazan and Yoram Singer},
title = {Adaptive subgradient methods for online learning and stochastic
         optimization},
year = 2011,
journal = jmlr,
volume = {12},
pages = {2121--2159},
}

@inproceedings{DuchiJoWa13_focs,
author = {John C. Duchi and Michael I. Jordan and Martin J. Wainwright},
title = {Local privacy and statistical minimax rates},
booktitle = focs13,
year = 2013,
pages = {429--438},
}

@article{DuchiJoWa18,
author = {John C. Duchi and Michael I. Jordan and Martin J. Wainwright},
title = {Minimax Optimal Procedures for Locally Private Estimation
 (with discussion)},
year = 2018,
journal = jasa,
volume = 113,
number = 521,
pages = {182--215},
}

@inproceedings{DworkKeMcMiNa06,
author = {Cynthia Dwork and Krishnaram Kenthapadi and Frank McSherry
  and Ilya Mironov and Moni Naor},
title = {Our Data, Ourselves: Privacy Via Distributed Noise Generation},
booktitle = {Advances in Cryptology (EUROCRYPT 2006)},
year = 2006,
comment = { Defines (\epsilon, \delta)-differential privacy by allowing
  an additive probability of error if the communicated statistics are
  extremely unlikely, i.e. P(Z \in S \mid x) \le e^\epsilon P(Z \in S \mid x')
  + \delta, where Z is the released statistic and x and x' are different
  in the way that we want to make private.
  Also shows how to generate noise in a distributed way.},
}

@inproceedings{DworkMcNiSm06,
author = {Cynthia Dwork and Frank McSherry and Kobbi Nissim and Adam Smith},
title = {Calibrating noise to sensitivity in private data analysis},
year = 2006,
booktitle = {Proceedings of the Third Theory of Cryptography Conference},
pages = {265--284},
comment = {
The initial paper that defines differential privacy as a ratio of probabilities
of outputs of statistical procedures on datasets differing in at most
one entry. They show that adding Laplace-distributed noise guarantees
privacy. They also give a few lower bounds (which seem quite loose and
also a bit hard to evaluate) that show that an \epsilon-private mechanism
cannot answer a large number of queries/statistical questions unless
the number of data points is large. The notion of answering queries is
a little cryptographic though. Appendices show the
equivalence of some semantic notions of differential privacy; it seems
that differential privacy means that an adversary knowing most of a dataset
cannot recover information about a single entry in a dataset with a few
queries of the dataset.
},
}

@article{DworkRo14,
 author = {Dwork, Cynthia and Roth, Aaron},
 title = {The Algorithmic Foundations of Differential Privacy},
 journal = {Foundations and Trends in Theoretical Computer Science},
 volume = {9},
 number = {3 \& 4},
 year = {2014},
 pages = {211--407},
 numpages = {197},
 publisher = {Now Publishers Inc.},
 address = {Hanover, MA, USA},
}

@inproceedings{FeldmanKoTa20,
author = {Vitaly Feldman and Tomer Koren and Kunal Talwar},
title = {Private Stochastic Convex Optimization: Optimal Rates in Linear Time},
year = 2020,
booktitle = stoc20,
}

@book{HiriartUrrutyLe93ab,
author = {J. Hiriart-Urruty and C. Lemar\'echal},
title = {Convex {A}nalysis and {M}inimization {A}lgorithms {I} \& {II}},
publisher = {Springer},
address = {New York},
year = 1993,
comment = {Builds convex analysis and calculus from first principles,
           lots of subgradient calculus},
}

@article{HochreiterJu97,
  title={Long short-term memory},
  author={Hochreiter, Sepp and Schmidhuber, J{\"u}rgen},
  journal={Neural computation},
  volume={9},
  number={8},
  pages={1735--1780},
  year={1997},
  publisher={MIT Press}
}

@article{KairouzRiRuTh20,
  title={Dimension independence in unconstrained private {ERM} via adaptive preconditioning},
  author={Peter Kairouz and M{\'o}nica Ribero and Keith Rush and Abhradeep Thakurta},
  journal={arXiv:2008.06570 [cs.LG]},
  year={2020}
}

@inproceedings{LevyDu19,
author = {Daniel Levy and John C. Duchi},
year = 2019,
title = {Necessary and Sufficient Geometries for Gradient Methods},
booktitle = nips2019,
url = {https://arxiv.org/abs/1909.10455},
}

@inproceedings{McMahanSt10,
author = {Brendan McMahan and Matthew Streeter},
title = {Adaptive Bound Optimization for Online Convex Optimization},
booktitle = colt10,
year = 2010,
}

@inproceedings{MerityXiBrSo17,
author = {Stephen Merity and Caiming Xiong and James Bradbury and Richard Socher},
year = 2017,
title = {Pointer Sentinel Mixture Models},
booktitle = iclr17,
}

@article{NemirovskiJuLaSh09,
author = {A. Nemirovski and A. Juditsky and G. Lan and A. Shapiro},
title = {Robust stochastic approximation approach to stochastic programming},
year = 2009,
journal = siopt,
volume = 19,
number = 4,
pages = {1574--1609},
}

@misc{PichapatiSYRK20,
title = "{A}da{C}li{P}: Adaptive clipping for private {SGD}",
notes = "TPDP workshop",
year = "2020",
author = "Pichapati, Venkatadheeraj  and Suresh, Ananda Theertha  and Yu, Felix X.  and Reddi, Sashank J.  and  Kumar, Sanjiv",
}

@inproceedings{TalwarThZh15,
  title={Nearly optimal private lasso},
  author={Kunal Talwar and Abhradeep Thakurta and Li Zhang },
  booktitle= nips2015,
  pages={3025--3033},
  year= 2015
}

@inproceedings{SmithTh13lasso,
author = {Adam Smith and Abhradeep Thakurta},
title = {Differentially Private Feature Selection via Stability Arguments,
 and the Robustness of the {L}asso},
year = 2013,
booktitle = colt13,
pages = {819--850},
url = {http://proceedings.mlr.press/v30/Guha13.html},
}

@inproceedings{SmithTh13,
author = {Adam Smith and Abhradeep Thakurta},
title = {({N}early) optimal algorithms for private online
  learning in full-information and bandit settings},
year = 2013,
booktitle = nips2013,
}

@book{Vershynin19,
author = {Roman Vershynin},
title = {High Dimensional Probability: An Introduction with Applications
  in Data Science},
year = 2019,
publisher = {Cambridge University Press},
}

@inproceedings{YuZCL21,
title={Do not Let Privacy Overbill Utility:  Gradient Embedding Perturbation for Private Learning},
author={Da Yu and Huishuai Zhang and Wei Chen and Tie-Yan Liu},
booktitle={International Conference on Learning Representations},
year={2021},
url={https://openreview.net/forum?id=7aogOj_VYO0}
}

@article{ZhouWuBa20,
  title={Bypassing the ambient dimension: Private {SGD} with gradient subspace identification},
  author={Yingxue Zhou and Zhiwei Steven Wu and Arindam Banerjee},
  journal={arXiv:2007.03813 [cs.LG]},
  year={2020}
}

@inproceedings{Zinkevich03,
 author = {Martin Zinkevich},
 title = "Online Convex Programming and Generalized Infinitesimal
		Gradient Ascent",
 booktitle = icml03,
 year = 2003,
 comment = {Simple, clear proof that online gradient descent achieves a
  regret bound of O(\sqrt{T}). Analysis is quite similar to that used
  in Boyd's lecture notes on the subgradient method as well as the
  Robust Stochastic Approximation paper by Nemirovski et al.
},
}
\newpage

\newpage
\appendix

\noindent
{\Large{\textbf{Appendix}}}

\section{Convergence of SGD and AdaGrad with biased gradients estimates}\label{sec:Biased-SGD-Adagrad}
For the sake of our analysis, we find it helpful to first study the convergence of SGD and AdaGrad when the stochastic estimates of the subgradients may be biased and noisy (Algorithms \ref{Algorithm3} and \ref{Algorithm4}.)

\begin{algorithm}
    \caption{Biased SGD}
    \label{Algorithm3}
    \begin{algorithmic}[1]
    \REQUIRE Dataset $\mathcal{S} = (\ds_1,\dots,\ds_n) \in \domain^n$, convex set $\xdomain$, mini-batch size $\batch$, number of iterations $T$.
    \STATE Choose arbitrary initial point $x^0 \in \xdomain$;
    \FOR{$k = 0;\ k \leq T-1;\ k = k + 1$}
        \STATE Sample a batch $\mathcal{D}_k:= \{z_i^k\}_{i=1}^\batch$ from $\mathcal{S}$ uniformly with replacement;
        \STATE Set $g^k := \frac{1}{\batch} \sum_{i=1}^\batch g^{k,i}$ where $g^{k,i} \in \partial F(x^k; z_i^k)$;
        \STATE Set $ \tilde g^k$ be the biased estimate of $g^k$;
        \STATE Set $ \hg^k := \tilde g^k + \noise^k $ where $\noise^k$ is a zero-mean random variable, independent from previous information;
        \STATE $x^{k+1} :=  \proj_{\xdomain}( x^k -  \alpha_k \hg^k)$;
    \ENDFOR 
    \STATE {\bfseries Return:} $\wb{x}^T \defeq \frac{1}{T} \sum_{k = 1}^T x^k$.
    \end{algorithmic}
\end{algorithm}
\begin{algorithm}
    \caption{Biased Adagrad}
    \label{Algorithm4}
    \begin{algorithmic}[1]
    \REQUIRE Dataset $\mathcal{S} = (\ds_1,\dots,\ds_n) \in \domain^n$, convex set $\xdomain$, mini-batch size $\batch$, number of iterations $T$.
    \STATE Choose arbitrary initial point $x^0 \in \xdomain$;
    \FOR{$k = 0;\ k \leq T-1;\ k = k + 1$}
        \STATE Sample a batch $\mathcal{D}_k:= \{z_i^k\}_{i=1}^\batch$ from $\mathcal{S}$ uniformly with replacement;
        \STATE Set $ \tilde g^k$ be the biased estimate of $g^k$;
        \STATE Set $ \hg^k := \tilde g^k + \noise^k $ where $\noise^k$ is a zero-mean random variable, independent from previous information;
        \STATE Set $H_k = \diag\left(\sum_{i=1}^{k} \hg^i \hg^{i^T}  \right)^{\frac{1}{2}}/ \diam_\infty(\xdomain)$;
        \STATE $x^{k+1} =  \proj_{\xdomain}( x^k -  H_k^{-1}  \hg^k)$ where the projection is with respect to $\norm{\cdot}_{H_k}$;
    \ENDFOR 
    \STATE {\bfseries Return:} $\wb{x}^T \defeq \frac{1}{T} \sum_{k = 1}^T x^k$.\\
    \end{algorithmic}
\end{algorithm}
Also, let
\begin{equation*}
\bias_{\norm{\cdot}}(\tilde g^k) = \E_{\mathcal{D}_k}\left[ \norm{\tilde g^k - g^k} \right]
\end{equation*}
be the bias of $\tilde g_k$ with respect to a general norm $\norm{\cdot}$. The next two theorems characterize the convergence of these two algorithms using this term.
\begin{theorem}
  \label{theorem:biased-sgd}
  Consider the biased SGD method (Algorithm \ref{Algorithm3}) with a non-increasing sequence of stepsizes $\{\stepsize_k\}_{k=0}^{T-1}$. Then
  for any $x\opt \in \argmin_{\xdomain} f$, we have
  \begin{equation*}
    \E[f(\wb{x}^T) - f(x\opt)]
    \le \frac{\diam_2(\xdomain)^2}{2 T \stepsize_{T-1}}
    + \frac{1}{2T} \sum_{k = 0}^{T-1}  \E[ \stepsize_k \ltwo{\hg^k}^2]
    +\frac{\diam_{\dnorm{\cdot}}(\xdomain)}{T} \sum_{k=0}^{T-1} \bias_{\norm{\cdot}}(\tilde g^k).
  \end{equation*}
\end{theorem}
\begin{proof}
We first consider the progress of a single step of the gradient-projected
stochastic gradient method.
We have
\begin{align*}
  \half \ltwos{x^{k + 1} - x\opt}^2
  & \le \half \ltwos{x^k - x\opt}^2
  - \stepsize_k \<\hat{g}^k, x^k - x\opt\>
  + \frac{\stepsize_k^2}{2} \ltwos{\hat{g}^k}^2 \\
  & = \half \ltwos{x^k - x\opt}^2
  - \stepsize_k \<f'(x^k), x^k - x\opt\>
  + \stepsize_k E_k
  + \frac{\stepsize_k^2}{2} \ltwos{\hat{g}^k}^2,
\end{align*}
where the error random variable $E_k$ is given by
\begin{equation*}
  E_k \defeq \<f'(x^k) - g^k, x^k - x\opt\>
  + \<g^k - \tilde{g}^k, x^k - x\opt\>
  + \<\tilde{g}^k - \hat{g}^k, x^k - x\opt\>.
\end{equation*}
Using that $\<f'(x^k), x^k - x\opt\>
\le f(x^k) - f(x\opt)$ then yields
\begin{equation*}
  f(x^k) - f(x\opt)
  \le \frac{1}{2 \stepsize_k}
  \left(\ltwos{x^k - x\opt}^2 - \ltwos{x^{k+1} - x\opt}^2\right)
  + \frac{\stepsize_k}{2} \ltwos{\hat{g}^k}^2
  + E_k.
\end{equation*}
Summing for $k = 0, \ldots, T-1$, by rearranging the terms and using that the stepsizes
are non-increasing, we obtain
\begin{equation}
  \label{eqn:intermediate-regret-bound}
  \sum_{k = 1}^T [f(x^k) - f(x\opt)]
  \le \frac{\diam(\xdomain)^2}{2 \stepsize_{T-1}}
  + \sum_{k = 0}^{T-1} \frac{\stepsize_k}{2} \ltwos{\hat{g}^k}^2
  + \sum_{k = 0}^{T-1} E_k.
\end{equation}
Taking expectations from both sides, we have
\begin{align*}
  \E[E_k] 
  &= \E \left[ \<f'(x^k) - g^k, x^k - x\opt\> \right ] + \E \left[ \<g^k - \tilde{g}^k, x^k - x\opt\> \right ] + \E \left[ \<\tilde{g}^k - \hat{g}^k, x^k - x\opt\> \right ]\\ 
  &= \E \left[ \<g^k - \tilde{g}^k, x^k - x\opt\> \right ] \\
  & \le \bias_{\norm{\cdot}}(\tilde g^k) \cdot \diam_{\dnorm{\cdot}}(\xdomain),
\end{align*}
 where the second equality comes from the fact that the two other expectations are zero and the last inequality follows from the Holder's inequality.
\end{proof}
\begin{remark}
This result holds in the case that $\stepsize_k$'s are adaptive and depend on observed gradients. 
\end{remark}

Next theorem states the convergence of biased Adagrad (Algorithm \ref{Algorithm4}).
\begin{theorem}
  \label{theorem:biased-adagrad}
  Consider the biased Adagrad method (Algorithm \ref{Algorithm4}). Then
  for any $x\opt \in \argmin_{\xdomain} f$, we have
  \begin{align*}
\E [f(\wb{x}^T) - f(x\opt)]
\leq  \frac{\diam_\infty(\xdomain)}{T} \sum_{j=1}^d \E \left [ \sqrt{\sum_{k=0}^{T-1} (\hg^{k}_j)^2} \right ]
+ \frac{\diam_{\dnorm{\cdot}}(\xdomain)}{T} \sum_{k=0}^{T-1} \bias_{\norm{\cdot}}(\tilde g^k).
\end{align*}
\end{theorem}
\begin{proof}
Recall that $x^{k+1}$ is the projection of $x^k - H_k^{-1} \hat{g}^k$ into $\xdomain$ with respect to $\Vert . \Vert_{\mathcal{H}_k}$. Hence, since $x \opt \in \xdomain$ and projections are non-expansive, we have
\begin{equation} \label{eqn:non-expansive}
\normH{ x^{k+1} - x \opt}{k}^2
\leq \normH{ x^k - H_k^{-1} \hg^k - x \opt}{k}^2.
\end{equation}
Now, expanding the right hand side yields
\begin{align*}
  \half \normH{x^{k + 1} - x\opt}{k}^2
  & \leq \half \normH{x^k - x\opt}{k}^2 -  \<\hg^k, x^k - x\opt\> + \half \normHi{\hg^k}{k}^2 \\
  & = \half \normH{x^k - x\opt}{k}^2
  - \<g^k, x^k - x\opt\>
  + \<g^k - \hg^k, x^k - x\opt\>
  + \half \normHi{\hg^k}{k}^2.
\end{align*}
Taking expectation and using that $ \E \left [\<g^k, x^k - x\opt\> \right ] \geq \E \left [ f(x^k) - f(x\opt) \right]$ from convexity along with the fact that $\E \left [ \<g^k - \hg^k, x^k - x\opt\> \right ] = \E \left [ \<g^k - \tilde g^k, x^k - x\opt\> \right ] $, we have
\begin{align*}
\half \E & \left [  \normH{x^{k + 1} - x\opt}{k}^2 \right ] \nonumber \\
& \leq \E \left [ \half \normH{x^k - x\opt}{k}^2 - (f(x^k) - f(x\opt))  + \half \normHi{\hg^k}{k}^2 \right ]  
+ \E \left [ \<g^k - \tilde g^k, x^k - x\opt\> \right ]. 
\end{align*}
Thus, using Holder's inequality, we have
\begin{equation*}
f(x^k) - f(x\opt)
	\le  \half  \E \left[  \normH{x^{k} - x\opt}{k}^2 - \normH{x^{k + 1} - x\opt}{k}^2 
		+\normHi{\hg^k}{k}^2 \right] 
		+  \bias_{\norm{\cdot}}(\tilde g^k) \cdot \diam_{\dnorm{\cdot}}(\xdomain).
\end{equation*}
Now the claim follows using standard techniques for Adagrad (as 
for example Corollary 4.3.8 in~\cite{Duchi18}).
\end{proof}


\section{Proofs of Section \ref{sec:algs}} \label{sec:proofs-alg}
\subsection{Proof of Lemma \ref{lemma:privacy}}\label{sec:proof-lemma:privacy}
The proof mainly follows from Theorem 1 in \cite{AbadiChGoMcMiTaZh16} where the authors provide a tight privacy bound for mini-batch SGD with bounded gradient using the Moments Accountant technique. Here we do not have the bounded gradient assumption. However, recall that we have
\begin{equation*}
\hg^k = \frac{1}{\batch} \sum_{i=1}^\batch \tilde{g}^{k,i} + \frac{\sqrt{\log(1/\delta)}}{\batch \diffp} \noise^k, \quad \tilde{g}^{k,i} = \pi_{A_k}(g^{k,i}),  
\end{equation*}
where $\norm{\tilde{g}^{k,i}}_{A_k} \leq 1$ and $\noise^k \simiid \normal(0, A_k^{-1})$. Note that for any Borel-measurable set $O \subset \reals^d$, $A_k^{1/2} O$ is also Borel-measurable, and furthermore, we have
\begin{align*}
\P\left(\hg^k \in O \right) &= \P\left(A_k^{1/2} \hg^k \in A_k^{1/2} O \right)  = \P\left(\frac{1}{\batch} \sum_{i=1}^\batch A_k^{1/2} \tilde{g}^{k,i} + \frac{\sqrt{\log(1/\delta)}}{\batch \diffp} A_k^{1/2} \noise^k \in A_k^{1/2} O \right),
\end{align*}
where, now, $\norm{A_k^{1/2} \tilde{g}^{k,i}}_2 \leq 1$ and $A_k^{1/2} \noise^k \simiid \normal(0, I_d)$ and we can use Theorem 1 in \cite{AbadiChGoMcMiTaZh16}. 
\subsection{The proof deferred from Example \ref{example-sub-Gaussian}} \label{sec:proof-example-sub-Gaussian}
Note that $\nabla F(x; z) = \nabla g(x) + Z$, and hence we could take $G(Z,C) = \sup_{x \in \xdomain} \|\nabla g(x)\|_C + \|Z\|_C $. As a result, by Minkowski inequality, we have
\begin{align}  
\E \left[ G(Z,C)^p \right ] ^{1/p}
	& \le \sup_{x \in \xdomain} \|\nabla g(x)\|_C + \E \left[ \|Z\|_C^{p} \right ] ^{1/p} 
	\le \mu + \sup_{x \in \xdomain} \E \left[ \|Z\|_C^{p} \right ] ^{1/p}. \label{eqn:example_1}
\end{align}
Now note that $C^{1/2} Z$ is $(C_{11}\sigma_1^2, \ldots, C_{dd}\sigma_d^2)$ sub-Gaussian. Also, 
we also know that if $X$ is $\sigma^2$ sub-gaussian, then
$\E[|X|^p]^{1/p} \le O(\sigma \sqrt{p})$, which implies the desired result.
\subsection{Intermediate Results}
Before discussing the proofs of Theorems \ref{theorem:convergence-SGD} and \ref{theorem:private-adagrad}, we need to state a few intermediate results which will be used in our analysis.

First, recall the definition of $\bias_{\norm{\cdot}}(\tilde g^k)$ from Section \ref{sec:Biased-SGD-Adagrad}:
\begin{equation*}
\bias_{\norm{\cdot}}(\tilde g^k) = \E_{\mathcal{D}_k}\left[ \norm{\tilde g^k - g^k} \right]
\end{equation*}
Here, we first bound the bias term. To do so, we use the following lemma:
\begin{lemma}[Lemma 3,~\cite{BarberDu14a}]
\label{lemma:projection-bias}
Consider the ellipsoid projection operator $\pi_D$. Then, for any random vector $X$ with $\E[\| X \|_C^p]^{1/p} \le \lipobj$, we have
\begin{equation*}
\E_X[\norm{\pi_D(X) - X}_C] \leq \frac{\lipobj^p}{(p - 1) B^{p-1}}.
\end{equation*}
\end{lemma}
We will find this lemma useful in our proofs. Another useful lemma that we will use it is the following:
  \begin{lemma}\label{lemma:sum_l2}
    Let $a_1, a_2, \ldots$ be an arbitrary sequence in $\R$. Let
    $a_{1:k} = (a_1, \ldots, a_i) \in \R^i$. Then
    \begin{equation*}
      \sum_{k = 1}^n \frac{a_k^2}{\ltwo{a_{1:k}}} \le
      2 \ltwo{a_{1:n}}.
    \end{equation*}
  \end{lemma}
  \begin{proof}
    We proceed by induction. The base case that $n = 1$ is immediate.
    Now, let us assume the result holds through index $n - 1$, and we wish to
    prove it for index $n$. The concavity of $\sqrt{\cdot}$ guarantees
    that $\sqrt{b + a} \le \sqrt{b} + \frac{1}{2 \sqrt{b}} a$, and so
    \begin{align*}
      \sum_{k = 1}^n \frac{a_k^2}{\ltwo{a_{1:k}}}
      & = \sum_{k = 1}^{n - 1} \frac{a_k^2}{\ltwo{a_{1:k}}}
      + \frac{a_n^2}{\ltwo{a_{1:n}}} \\
      & \leq 2 \ltwo{a_{1:n-1}}
      + \frac{a_n^2}{\ltwo{a_{1:n}}}
      = 2 \sqrt{\ltwo{a_{1:n}}^2 - a_n^2}
      + \frac{a_n^2}{\ltwo{a_{1:n}}} \\
      & \leq
      2 \ltwo{a_{1:n}},
    \end{align*}
    where the first inequality follows from the inductive hypothesis and the second one uses the concavity of $\sqrt{\cdot}$.
  \end{proof}
  
 \subsection{Proof of Theorem \ref{theorem:convergence-SGD}}
 \label{proof-theorem:convergence-SGD}
We first state a more general version of the theorem here:
\begin{theorem}
\label{theorem:convergence-SGD(general)}
Let $\Ds$ be a dataset with $n$ points sampled from distribution $P$. Let $C$ also be a diagonal and positive definite matrix. Consider running Algorithm \ref{Algorithm1} with $T=cn^2/b^2$, $A_k = C/B^2$ where $B > 0$ is a positive real number and $c$ is given by Lemma \ref{lemma:privacy}.  Then, with probability $1-1/n$, we have
\begin{align*}
\E  [f(\wb{x}^T;\Ds)  - \min_{x \in \xdomain} f(x;\Ds)]  
	& \leq \bigO(1)  \left( \frac{\diam_2(\xdomain)}{T} \sqrt{\sum_{k=1}^T \E [\ltwos{g^k}^2] } \right. \\
 	& \left. \quad +  \frac{\diam_2(\xdomain) B \sqrt{\tr(C^{-1})\log(1/\delta)} }{n \diffp}  
 	+   \frac{{ \diam_{\norm{\cdot}_{C^{-1}}}(\xdomain) } ~ (2\lipobj_{2p}(C))^p}{(p - 1) B^{p-1}} \right),
\end{align*}
where the expectation is taken over the internal randomness of the algorithm.
\end{theorem}
\begin{proof}
Let $x\opt \in \argmin _{x \in \xdomain} f(x;\Ds)$. Also, for simplicity, we suppress the dependence of $f$ on $\Ds$ throughout the proof.
First, by Lemma \ref{lemma:empirical_Lipschitz}, we know that with probability at least $1-1/n$, we have
\begin{equation*}
\hat{\lipobj}_p(\Ds;C) \leq 2 \lipobj_{2p}(C), 
\end{equation*}
We consider the setting that this bound holds.
Now, note that by Theorem \ref{theorem:biased-sgd} we have
\begin{equation}\label{eqn:SGD_initial_bound}
    \E[f(\wb{x}^T) - f(x\opt)]
    \le \frac{\diam_2(\xdomain)^2}{2 T \stepsize_{T-1}}
    + \frac{1}{2T} \sum_{k = 0}^{T-1} \E[\stepsize_k \ltwo{\hg^k}^2]
    + \frac{\diam_{\norm{\cdot}_{C^{-1}}}(\xdomain)}{T} \sum_{k = 0}^{T-1} \bias_{\norm{\cdot}_C}(\tilde g^k).
\end{equation}
Using Lemma \ref{lemma:projection-bias}, we immediately obtain the following bound
\begin{equation}\label{eqn:SGD_bias_bound} 
	\bias_{\norm{\cdot}_{C}}(\tilde g^k) 
		= \E \left[ \norm{\tilde g^k - g^k}_{C} \right] 
		\le \frac{\hat{\lipobj}_p(\Ds;C)^p}{(p-1) B^{p-1}}
		\le \frac{(2 \lipobj_{2p}(C))^p}{(p-1) B^{p-1}}
\end{equation}
Plugging \eqref{eqn:SGD_bias_bound} into \eqref{eqn:SGD_initial_bound}, we obtain
\begin{equation}\label{eqn:SGD_2_bound}
    \E[f(\wb{x}^T) - f(x\opt)]
    \le \frac{\diam_2(\xdomain)^2}{2 T \stepsize_{T-1}}
    + \frac{1}{2T} \sum_{k = 0}^{T-1} \E[\stepsize_k \ltwo{\hg^k}^2]
    + \frac{{ \diam_{\norm{\cdot}_{C^{-1}}}(\xdomain) } ~ (2\lipobj_{2p}(C))^p}{(p - 1)B^{p-1}}.
\end{equation}
Next, we substitute the value of $\alpha_k$ and use Lemma \ref{lemma:sum_l2} to obtain
\begin{equation*}
 \sum_{k = 0}^{T-1} \E[\stepsize_k \ltwo{\hg^k}^2] \leq 2 \diam_2(\xdomain) \sqrt{\sum_{k=1}^T \E [\ltwos{\hg^k}^2] },   
\end{equation*}
and by replacing it in \eqref{eqn:SGD_2_bound}, we obtain 
\begin{equation}\label{eqn:SGD_3_bound}
    \E[f(\wb{x}^T) - f(x\opt)]
    \le \frac{3\diam_2(\xdomain)}{2T} \sqrt{\sum_{k=1}^T \E [\ltwos{\hg^k}^2] }
    + \frac{{ \diam_{\norm{\cdot}_{C^{-1}}}(\xdomain) } ~ (2\lipobj_{2p}(C))^p}{(p - 1)B^{p-1}}.
\end{equation}
Finally, note that
\begin{align}
\sqrt{\sum_{k=1}^T \E [\ltwos{\hg^k}^2] } 
& = \sqrt{\sum_{k=1}^T \E [\ltwos{\tg^k}^2] + \frac{\log(1/\delta)}{\batch^2 \epsilon^2} \sum_{k=0}^{T-1} \tr(A_k^{-1})} \nonumber \\
& = \sqrt{\sum_{k=1}^T \E [\ltwos{\tg^k}^2] + T \frac{B^2 \log(1/\delta) \tr(C^{-1})}{\batch^2 \epsilon^2}} \nonumber \\
& \leq \sqrt{2} \left ( \sqrt{\sum_{k=1}^T \E [\ltwos{\tg^k}^2] } 
  + \frac{B \sqrt{\log(1/\delta) \tr(C^{-1})}}{\batch \epsilon} \sqrt{T} \right), \label{eqn:SGD_4_bound}
\end{align}
where the last inequality follows from the fact that $\sqrt{x+y} \leq \sqrt{2} \left ( \sqrt{x} + \sqrt{y} \right)$ for nonnegative real numbers $x$ and $y$. Plugging \eqref{eqn:SGD_4_bound} into \eqref{eqn:SGD_3_bound} completes the proof. 
\end{proof}
\subsection{Proof of Theorem \ref{theorem:private-adagrad}} \label{proof-theorem:private-adagrad}
We first state the more general version of theorem:
\begin{theorem}\label{theorem:private-adagrad(general)}
Let $\Ds$ be a dataset with $n$ points sampled from distribution $P$. Let $C$ also be a diagonal and positive definite matrix. Consider running Algorithm \ref{Algorithm1} with $T=cn^2/b^2$ $A_k = C/B^2$ where $B > 0$ is a positive real number and $c$ is given by Lemma \ref{lemma:privacy}. Then,
with probability $1-1/n$, we have
\begin{align*}
    \E  [f(\wb{x}^T;\Ds)  &- \min_{x \in \xdomain} f(x;\Ds)]
     \leq  \bigO(1) \left ( \frac{ \diam_\infty(\xdomain)}{T} \sum_{j=1}^d 
    \E \left [ \sqrt{\sum_{k=1}^T (g^{k}_j)^2} \right ] \right. \\ 
    & \left. \quad +  \frac{\diam_\infty(\xdomain) B \sqrt{\log(1/\delta)} (\sum_{j=1}^d C_{jj}^{-\half})}{n \diffp } 
    + \frac{ \diam_{\norm{\cdot}_{C^{-1}}}(\xdomain) (2\lipobj_{2p}(C))^p}{(p-1) B^{p-1}} \right),
\end{align*}
where the expectation is taken over the internal randomness of the algorithm.
\end{theorem}
\begin{proof}
Similar to the proof of Theorem \ref{theorem:convergence-SGD}, we choose $x\opt \in \argmin _{x \in \xdomain} f(x;\Ds)$. We suppress the dependence of $f$ on $\Ds$ throughout this proof as well.
Again, we focus on the case that the bound
\begin{equation*}
\hat{\lipobj}_p(\Ds;C) \leq 2 \lipobj_{2p}(C), 
\end{equation*}
which we know its probability is at least $1-1/n$.

Using Theorem \ref{theorem:biased-adagrad}, we have
  \begin{align*}
\E [f(\wb{x}^T) - f(x\opt)]
\leq  \frac{\diam_\infty(\xdomain)}{T} \sum_{j=1}^d \E \left [ \sqrt{\sum_{k=0}^{T-1} (\hg^{k}_j)^2} \right ]
 + \frac{\diam_{\norm{\cdot}_{C^{-1}}}(\xdomain)}{T} \sum_{k = 0}^{T-1} \bias_{\norm{\cdot}_C}(\tilde g^k).  
\end{align*}
Similar to the proof of Theorem \ref{theorem:convergence-SGD}, and by using Lemma \ref{lemma:projection-bias}, we could bound the second term with
\begin{equation*}
\frac{ \diam_{\norm{\cdot}_{C^{-1}}}(\xdomain) (2\lipobj_{2p}(C))^p}{(p-1) B^{p-1}}.    
\end{equation*}
Now, it just suffices to bound the first term. Note that
\begin{align}
\sum_{j=1}^d \E \left [ \sqrt{\sum_{k=1}^T (\hg^{k}_j)^2} \right ]
& = \sum_{j=1}^d \E \left [ \sqrt{\sum_{k=1}^T (\tg^{k}_j + \noise^k_j )^2} \right ] \nonumber \\
& \leq \sum_{j=1}^d \E \left [ \sqrt{\sum_{k=1}^T 2  \left( (\tg^{k}_j)^2 + (\noise^k_j )^2 \right )} \right ] \nonumber \\
& \leq 2 \sum_{j=1}^d \left ( \E \left [ \sqrt{\sum_{k=1}^T  (\tg^{k}_j)^2} \right ] + \E \left [\sqrt{\sum_{k=1}^T (\noise^k_j )^2} \right ] \right ) \label{eqn:AdaGrad-regret-dpnoise1} \\
& \leq 2 \sum_{j=1}^d \left ( \E \left [ \sqrt{\sum_{k=1}^T  (\tg^{k}_j)^2} \right ] + \sqrt{\E \left [\sum_{k=1}^T (\noise^k_j )^2 \right ]} \right ) \label{eqn:AdaGrad-regret-dpnoise2} \\
& \leq 2 \sum_{j=1}^d  \E \left [ \sqrt{\sum_{k=1}^T  (\tg^{k}_j)^2} \right ] 
+ 2 B \sqrt{T \log(1/\delta)} \frac{\sum_{j=1}^d C_{jj}^{-1/2}}{\batch \epsilon},
\label{eqn:AdaGrad-regret-dpnoise3}
\end{align}
where \eqref{eqn:AdaGrad-regret-dpnoise1} is obtained by using $\sqrt{x+y} \leq \sqrt{2} \left ( \sqrt{x} + \sqrt{y} \right)$ with $x= \sum_{k=1}^T  (\tg^{k}_j)^2$ and $y=\sum_{k=1}^T (\noise^k_j )^2$, and \eqref{eqn:AdaGrad-regret-dpnoise2} follows from $\E \left [ X \right ] \leq \sqrt{\E \left [ X^2 \right ]}$ with $X = \sqrt{\sum_{k=1}^T (\noise^k_j )^2}$.


\end{proof}

\section{Proof of Theorem~\ref{thm:unknown-cov}}
\label{sec:proof-unknown-cov}
We begin with the following lemma, which upper bounds the bias from truncation.
\begin{lemma}
  \label{lemma:trunc-bias}
  Let $Z$ be a random vector
  satisfying Definition~\ref{definition:bouned-moments-ratio}. Let $ \sigma_j^2 = \E[z_j^2]$ and $\Delta \ge 4 r \sigma_j \log r$. Then we have
    \begin{equation*}
    |\E[\min(z_j^2,\Delta^2)] - \E[z_j^2] | \le  \sigma_j^2/8 .
    \end{equation*}
\end{lemma}

\begin{proof}
    Let $\sigma_j^2 = \E[z_j^2]$.
    To upper bound the bias, we need to upper bound $P(z_j^2 \ge t \Delta^2)$. We have that $z_j$ is $r^2\sigma_j^2$-sub-Gaussian therefore
    \begin{equation*}
    P(z_j^2 \ge t r^2 \sigma_j^2) \le 2 e^{-t}.
    \end{equation*}
    Thus, if $Y = |\min(z_j^2,\Delta^2) - z_j^2|$ then
    $P(Y \ge t r^2 \sigma_j^2) \le 2 e^{-t}$ hence
    \begin{align*}
    \E[Y] 
        & = \int_{0}^\infty P(Y \ge t) dt \\
        & = \int_{0}^\infty P(z_j^2 \ge \Delta^2 + t) dt \\
        & \le \int_{0}^\infty 2 e^{-(\Delta^2 + t)/r^2 \sigma_j^2} dt \\
        & \le 2 r^2 \sigma_j^2 e^{-\Delta^2/r^2 \sigma_j^2}
        \le \sigma_j^2/8,
    \end{align*}
    where the last inequality follows since $\Delta = 4 r \sigma_j \log r$.
\end{proof}

The following lemma demonstrates that the random variable $Y_i = \min(z_{i,j}^2,\Delta^2)$ quickly concentrates around its mean.
\begin{lemma}
  \label{lemma:sub-exp-conc}
  Let $Z$ be a random vector satisfying
  Definition~\ref{definition:bouned-moments-ratio}. Then with probability at
  least $1- \beta$,
  \begin{equation*}
    \left|\frac{1}{n} \sum_{i=1}^n  \min(z_{i,j}^2,\Delta^2) - \E[\min(z_{j}^2,\Delta^2)]\right|
    \le \frac{2 r^2 \sigma_j^2 \sqrt{\log(2/\beta)}}{\sqrt{n}}.
  \end{equation*}
\end{lemma}
\begin{proof}
  Let $Y_i = \min(z_{i,j}^2,\Delta^2)$. Since $z_j$ is
  $r^2\sigma_j^2$-sub-Gaussian, we get that $z_j^2$ is
  $r^4\sigma_j^4$-sub-exponential, meaning that $\E[(z_j^2)^k]^{1/k} \le
  \bigO(k) r^2 \sigma_j^2$ for all $k \ge 1$.  Thus $Y_i$ is also
  $r^4\sigma_j^4$-sub-exponential, and using Bernstein's
  inequality~\citep[Theorem 2.8.1]{Vershynin19}, we obtain
  \begin{equation*}
    P\left(\left|\frac{1}{n} \sum_{i=1}^n  Y_i - \E[Y_i]\right| \ge t \right)
    \le 2 \exp \left(- \min \left\{ \frac{n t^2}{2r^4 \sigma_j^4},
    \frac{n t}{r^2 \sigma_j^2}\right\} \right) .   
  \end{equation*}
  Setting $t = r^2 \sigma_j^2 \frac{2 \sqrt{\log(2/\beta)}}{\sqrt{n}} $
  yields the result.
\end{proof}

Given Lemmas~\ref{lemma:trunc-bias} and~\ref{lemma:sub-exp-conc},
we are now ready to finish the proof of~\Cref{thm:unknown-cov}.

\paragraph{Proof of Theorem~\ref{thm:unknown-cov}}
First, privacy follows immediately, as each iteration $t$ is
$(\diffp/T,\delta/T)$-DP (using standard properties of the Gaussian
mechanism~\cite{DworkRo14}), so basic composition implies that the final
output is $(\diffp,\delta)$-DP. We now proceed to prove the claim about
utility. Let $\rho_t^2$ be the truncation value at iterate $t$, i.e.,
$\rho_t = 4 r \log r/2^{t-1}$.  First, note that~\Cref{lemma:sub-exp-conc}
implies that with probability $1-\beta/2$ for every $j \in [d]$
\begin{equation*}
  \left| \frac{1}{n} \sum_{i=1}^n \min(z_{i,j}^2, \rho_t^2) - \E[\min(z_{j}^2,\rho_t^2)] \right|  
  \le \frac{2 r^2 \sigma_j^2 \sqrt{\log(8 d/\beta)}}{\sqrt{n}} \le \sigma_j^2 / 10,
\end{equation*}
and similar arguments show that
\begin{equation*}
  \left| \sigma_j^2 - \frac{1}{n} \sum_{i=1}^n z_{i,j}^2\right|
  \le \frac{2 r^2 \sigma_j^2 \sqrt{\log(8 d/\beta)}}{\sqrt{n}} \le \sigma_j^2 / 10,
\end{equation*}
where the last inequality follows since $n \ge 400 r^4 \log(8d/\beta)$.
Moreover, for $\sigma_j$ such that $\rho_t \ge 4 r \sigma_j \log r$,
\Cref{lemma:trunc-bias} implies that
\begin{equation*}
  | \E[\min(z_{j}^2,\rho_t^2) - \sigma_j^2 | \le  \sigma_j^2/8.
\end{equation*}
Let us now prove that if $\sigma_j = 2^{-k}$ then its value will be set at
most at iterate $t = k$. Indeed at iterate $t=k$ we have $\rho_t = 4 r
2^{-k} \log r \ge 4 r \sigma_j \log r$ hence we have that using the triangle
inequality and standard concentration resutls for Gaussian distributions
that with probability $1- \beta/2$
\begin{equation*}
  | \hat \sigma_{k,j}^2 - \sigma_j^2|
  \le \sigma_j^2 / 5 + \frac{16 r^2 T \sqrt{d} \log^2 r \log(T/\delta) \log(4 d/\beta) }{ 2^{2k} n \diffp}
  \le \sigma_j^2/4,
\end{equation*}
where the last inequality follows since $n \diffp \ge 1000 r^2 T \sqrt{d}
\log^2 r \log(T/\delta) \log(4 d/\beta)$. Thus, in this case we get that
$\hat \sigma_{k,j}^2 \ge \sigma_j^2/2 \ge 2^{-k-1}$ hence the value of
coordinate $j$ will best set at most at iterate $k$ hence $\hat \sigma_j \ge
\sigma_j / 2$.
    
On the other hand, we now assume that $\sigma_j = 2^{-k}$ and show that the
value of $\hat \sigma_j$ cannot be set before the iterate $t= k-3$ and hence
$\hat \sigma_j \le 2^{-k+3} \le 8 \sigma_j$. The above arguments show that
at iterate $t$ we have $\hat \sigma_{t,j}^2 \le 3/2 \sigma_j^2 + \frac{1}{10
  \cdot 2^{2k}} \le 2^{-2k +1} + \frac{1}{10 \cdot 2^{2k}} \le 2^{-2k+2}$
hence the first part of the claim follows.
    
To prove the second part, first note that $z_j$ is $r
\sigma_j$-sub-Gaussian, hence using~\Cref{theorem:private-adagrad}, it is
enough to show that $\lipobj_{2p}(\hat C) \le O(\lipobj_{2p}(C))$ and that
$\sum_{j=1}^d \hat C_j^{-1/2} \le O(1) \cdot \sum_{j=1}^d C_j^{-1/2}$ where
$C = (r \sigma_j)^{-4/3}$ is the optimal choice of $C$ as in the
bound~\eqref{eqn:pagan-subgaussian-bound}.  The first condition immediately
follows from the definition of $\lipobj_{2p}$ since $\hat C_j \le C_j$ for
all $j \in [d]$.  The latter condition follows immediately since $\frac{1}{2}
\max (\sigma_j,1/d^2) \le \hat \sigma_j $, implying
\begin{equation*}
  \sum_{j=1}^d \hat C_j^{-1/2}
  \le O(r^{-2/3}) \sum_{j=1}^d \hat \sigma_j^{-2/3}
  \le O(r^{-2/3}) \sum_{j=1}^d \sigma_j^{-2/3} + 1/d
  \le O(r^{-2/3}) \sum_{j=1}^d \sigma_j^{-2/3}.
\end{equation*}

\section{Proofs of Section~\ref{sec:LB} (Lower bounds)}
\label{sec:proofs-LB}

\subsection{Proof of~\Cref{prop:sign-LB-var}}
\label{sec:proof-LB-signs}
\newcommand{\smax}{\sigma_{\max}}
\newcommand{\smin}{\sigma_{\min}}

We begin with the following lemma which gives 
a lower bound for the sign estimation problem when $\sigma_j = \sigma$ for all $j \in [d]$. \citet{AsiFeKoTa21} use similar result to prove lower bounds for private optimization over $\ell_1$-bounded domains. For completeness, we give a proof in Section~\ref{sec:proof-LB-signs-identity}.
\begin{lemma}
\label{prop:sign-LB-identity}
	Let $\mech$ be $(\diffp,\delta)$-DP and $\Ds = (z_1,\dots,z_n) $ 
	where $z_i \in \mc{Z} = \{-\sigma,\sigma \}^d $.
	Then
	\begin{equation*}
	\sup_{\Ds \in \mc{Z}^n } 
	\E\left[ \sum_{j=1}^d |\bar{z}_j| 
		\indic {\sign(\mech_j(\Ds)) \neq \sign(\bar{z}_j)} \right] 
		\ge \min \left( \sigma d , \frac{\sigma d^{3/2} }{n \diffp \log d} \right).
	\end{equation*} 
\end{lemma}
We are now ready to complete the proof of~\Cref{prop:sign-LB-var}
using bucketing-based techniques. 
First, we assume without loss of generality that 
$\sigma_j \le 1$ for all $1 \le j \le d$ (otherwise 
we can divide by $\max_{1 \le j \le d} \sigma_j$).
Now, we define  
buckets of coordinates $B_0,\dots,B_K$ such that
\begin{equation*}
	B_i = \{j : 2^{-i-1} \le \sigma_j \le 2^{-i} \}.
\end{equation*}
For $i=K$, we set $B_K =  \{j : \sigma_j \le 2^{-K} \}$.
We let $\smax(B_i) = \max_{j \in B_i} \sigma_j$ denote the maximal
value of $\sigma_j$ inside $B_i$. Similarly,
we define $\smin(B_i) = \min_{j \in B_i} \sigma_j$. 
Focusing now on the $i$'th bucket, since 
$\sigma_j \ge \smin(B_i) $ for all $j \in B_i$, ~\Cref{prop:sign-LB-identity}
now implies (as $d \log^2 d \le (n\diffp)^2 $) the lower bound
\begin{equation*}
	\sup_{\Ds \in \mc{Z}^n } 
	\E\left[ \sum_{j \in B_i} |\bar{z}_j| 
		\indic {\sign(\mech_j(\Ds)) \neq \sign(\bar{z}_j)} \right] 
		\ge \frac{\smin(B_i) |B_i|^{3/2} }{n \diffp \log d} .
\end{equation*} 
Therefore this implies that 
\begin{equation*}
	\sup_{\Ds \in \mc{Z}^n } 
	\E\left[ \sum_{j =1}^d |\bar{z}_j| 
		\indic {\sign(\mech_j(\Ds)) \neq \sign(\bar{z}_j)} \right] 
		\ge \max_{0 \le i \le K} \frac{\smin(B_i) |B_i|^{3/2} }{n \diffp \log d} .
\end{equation*} 
To finish the proof of the theorem, it is now enough to prove that 
\begin{equation*}
	\sum_{j=1}^d \sigma_j^{2/3} \le 
		O(1) ~ \log d \max_{0 \le i \le K} {\smin(B_i)^{2/3} |B_i|}.
\end{equation*} 
We now have
\begin{align*}
	\sum_{j=1}^{d} \sigma_j^{2/3}
		& \le \sum_{i=0}^K |B_i| \smax(B_i)^{2/3} \\
		& \le K \max_{0 \le i \le K-1} |B_i| \smax(B_i)^{3/2} \\
		& \le 4 K \max_{0 \le i \le K-1} |B_i| \smin(B_i)^{3/2},
\end{align*} 
where the second inequality follows since the maximum cannot
be achieved for $i=K$ given our choice of $K = 10 \log d$,
and the last inequality follows since $\smax(B_i) \le 2 \smin(B_i)$
for all $i \le K - 1$.
This proves the claim.

\subsection{Proof of~\Cref{prop:sign-LB-identity}}
\label{sec:proof-LB-signs-identity}
Instead of proving lower bounds on the error of private mechanisms,
it is more convenient for this result to prove lower bounds on the 
sample complexity required to achieve a certain error.
Given a mechanism $\mech$ and data $\Ds \in \mc{Z}^n$, 
define the error of the mechanism to be:
\newcommand{\err}{\mathsf{Err}}
\begin{equation*}
	\err(\mech,\Ds) = 
	\E\left[ \sum_{j=1}^d |\bar{z}_j| 
		\indic {\sign(\mech_j(\Ds)) \neq \sign(\bar{z}_j)} \right] .
\end{equation*} 
The error of a mechanism for datasets of size $n$ is 
$\err(\mech, n) = \sup_{\Ds \in \mc{Z}^n} \err(\mech,\Ds)$.

We let $n\opt(\alpha,\diffp)$ denote the minimal $n$
such that there is an $(\diffp,\delta)$-DP (with $\delta = n^{-\omega(1)}$) mechansim $\mech$
such that $\err(\mech,n\opt(\alpha,\diffp)) \le \alpha$.
We prove the following lower bound on the sample
complexity.
\begin{proposition}
\label{prop:sample-complexity-LB}
	If $\linf{z} \le 1$ then
	\begin{equation*}
	n\opt(\alpha,\diffp) 
		\ge \Omega(1) \cdot \frac{d^{3/2} }{\alpha \diffp \log d} .
	\end{equation*} 
\end{proposition}

To prove this result, we first state the following
lower bound for constant $\alpha$ and $\diffp$
which follows from Theorem 3.2 in~\cite{TalwarThZh15}.
\begin{lemma}[\citet{TalwarThZh15}, Theorem 3.2]
\label{lemma:sample-complexity-LB-low-accuracy}
	Under the above setting,
	\begin{equation*}
	n\opt(\alpha = d/4,\diffp = 0.1) 
		\ge \Omega(1) \cdot \frac{\sqrt{d}}{\log d} .
	\end{equation*} 
\end{lemma}

We now prove a lower bound on the sample
complexity for small values of $\alpha$ and $\diffp$ 
which implies Proposition~\ref{prop:sample-complexity-LB}.
\begin{lemma}
\label{lemma:low-to-high-accuracy}
    Let $\diffp_0 \le 0.1$.
	For $\alpha \le \alpha_0/2$ and $\diffp \le \diffp_0/2$,
	\begin{equation*}
	n\opt(\alpha,\diffp ) 
	\ge \frac{\alpha_0 \diffp_0}{\alpha \diffp}
		n\opt(\alpha_0, \diffp_0)  .
	\end{equation*} 
\end{lemma}
\begin{proof}
	Assume there exists an $(\diffp,\delta)$-DP mechanism
	$\mech$ such that $\err(\mech,n) \le \alpha$. Then
	we now show that there is $ \mech'$
	that is $(\diffp_0, \frac{2 \diffp_0}{\diffp}  \delta)$-DP 
	with $ n' = \Theta(\frac{\alpha \diffp}{\alpha_0 \diffp_0} n )$ such that
	$\err( \mech', n') \le \alpha_0 $. This proves the claim.
	Let us now show how to define $ \mech'$ given $\mech$.
	Let $k = \floor{\log (1+\diffp_0)/\diffp}$. For $\Ds' \in \mc{Z}^{n'}$,
	we define $\Ds$ to have $k$ copies of $\Ds'$ and 
	$(n - kn')/2$ users which have $z_i = (\sigma,\dots,\sigma)$ and $(n - kn')/2$ users which have $z_i = (-\sigma,\dots,-\sigma)$.
	Then we simply define $\mech'(\Ds') = \mech(\Ds)$.
	Notice that now we have
	\begin{equation*}
		\bar{z} = \frac{ k n'}{n}  \bar{z'} .
	\end{equation*}
	Therefore for a given $\Ds'$ we have that:
	\begin{equation*}
	\err(\mech', \Ds') = \frac{n}{k n'} \err(\mech,\Ds) \le \frac{n \alpha}{k n'}
	\end{equation*}
	Thus if $n' \ge \frac{2n \alpha}{k \alpha_0} $ then 
	\begin{equation*}
	\err(\mech', \Ds') \le \alpha_0.
	\end{equation*} 
	Thus it remains to argue for the privacy of $\mech'$.
	By group privacy, 
	$\mech'$ is $(k\diffp, \frac{e^{k \diffp} - 1}{e^\diffp- 1} \delta)$-DP,
	hence our choice of $k$ implies that $ k \diffp \le \diffp_0$ and
	$\frac{e^{k \diffp} - 1}{e^\diffp- 1} \delta \le \frac{2 \diffp_0}{\diffp} \delta$.
\end{proof}

\subsection{Proof of~\Cref{thm:LB-opt-l2-var}}
\label{sec:proof-LB-l2}
We assume without loss of generality that 
$\sigma_j \le 1$ for all $1 \le j \le d$ (otherwise 
we can divide by $\max_{1 \le j \le d} \sigma_j$).
We follow the bucketing-based technique we had in the proof
of~\Cref{prop:sign-LB-var}.
We define  
buckets of coordinates $B_0,\dots,B_K$ such that
\begin{equation*}
	B_i = \{j : 2^{-i-1} \le \sigma_j \le 2^{-i} \}.
\end{equation*}
For $i=K$, we set $B_K =  \{j : \sigma_j \le 2^{-K} \}$.
We let $\smax(B_i) = \max_{j \in B_i} \sigma_j$ denote the maximal
value of $\sigma_j$ inside $B_i$. Similarly,
we define $\smin(B_i) = \min_{j \in B_i} \sigma_j$. 
Focusing now on the $i$'th bucket, since 
$\sigma_j \ge \smin(B_i) $ for all $j \in B_i$, ~\Cref{prop:LB-opt-l2-identity}
now implies the lower bound
\begin{equation*}
 	\sup_{\Ds \in \mc{Z}^n } 
 		\E \left[ f(\mech(\Ds);\Ds) -  f(x\opt_\Ds;\Ds) \right] 
 			\ge \min \left( \smin(B_i) \sqrt{|B_i|} , \frac{|B_i| \smin(B_i)}{n \diffp} \right).
\end{equation*}
Since $ d \le (n \diffp)^2$, 
taking the maximum over buckets, we get that the error of 
any mechanism is lower bounded by:
\begin{equation*}
 	\sup_{\Ds \in \mc{Z}^n } 
 		\E \left[ f(\mech(\Ds);\Ds) -  f(x\opt_\Ds;\Ds) \right] 
 			\ge \max_{0 \le i \le K}  
			      \frac{|B_i| \smin(B_i)}{n \diffp} .
\end{equation*}
To finish the proof, we only need to show now that
\begin{equation*}
 	\frac{\sum_{j=1}^d \sigma_j }{\log d} 
		\le O(1) \max_{0 \le i \le K}   {|B_i| \smin(B_i)}.
\end{equation*}
Indeed, we have that
\begin{align*}
	\sum_{j=1}^d \sigma_j
		& \le \sum_{i=0}^K |B_i| \smax(B_i) \\
		& \le K \max_{0 \le i \le K-1} |B_i| \smax(B_i) \\
		& \le 2 K \max_{0 \le i \le K-1} |B_i| \smin(B_i),
\end{align*} 
where the second inequality follows since the maximum cannot
be achieved for $i=K$ given our choice of $K = 10 \log d$,
and the last inequality follows since $\smax(B_i) \le 2 \smin(B_i)$
for all $i \le K - 1$. The claim follows.

\end{document}